\documentclass[pdflatex,sn-mathphys-num]{sn-jnl}

\usepackage{graphicx}%
\usepackage{multirow}%
\usepackage{amsmath,amssymb,amsfonts}%
\usepackage{amsthm}%
\usepackage{mathrsfs}%
\usepackage[title]{appendix}%
\usepackage{xcolor}%
\usepackage{textcomp}%
\usepackage{manyfoot}%
\usepackage{booktabs}%
\usepackage{algorithm}%
\usepackage{algorithmicx}%
\usepackage{algpseudocode}%
\usepackage{listings}%

\usepackage[capitalize]{cleveref}

\theoremstyle{thmstyleone}%
\newtheorem{theorem}{Theorem}
\newtheorem{lemma}{Lemma} 

\theoremstyle{thmstyletwo}%

\theoremstyle{thmstylethree}%
\newtheorem{definition}{Definition}%

\begin{document}

\title[Article Title]{Fast and Scalable Semi-Supervised Learning for Multi-View Subspace Clustering}

\author[1]{\fnm{Huaming} \sur{Ling}}\email{linghm18@mails.tsinghua.edu.cn}

\author*[2]{\fnm{Chenglong} \sur{Bao}}\email{clbao@mail.tsinghua.edu.cn}

\author[3]{\fnm{Jiebo} \sur{Song}}\email{songjiebo@bimsa.cn}

\author*[2]{\fnm{Zuoqiang} \sur{Shi}}\email{zqshi@tsinghua.edu.cn}

\affil[1]{\orgdiv{Department of  Mathematical Sciences}, \orgname{Tsinghua University}, \orgaddress{\street{Haidian  District}, \city{Beijing}, \postcode{100084}, \state{Beijing}, \country{China}}}

\affil*[2]{\orgdiv{Yau Mathematical Sciences Center}, \orgname{Tsinghua University}, \orgaddress{\street{Haidian  District}, \city{Beijing}, \postcode{100084}, \state{Beijing}, \country{China}}}

\affil[3]{\orgdiv{Beijing Institute of Mathematical Sciences and Applications}, \orgaddress{\street{Huairou  District}, \city{Beijing}, \postcode{101408}, \state{Beijing}, \country{China}}}


\abstract{In this paper, we introduce a Fast and Scalable Semi-supervised Multi-view Subspace Clustering (FSSMSC) method, a novel solution to the high computational complexity commonly found in existing approaches. FSSMSC features linear computational and space complexity relative to the size of the data. The method generates a consensus anchor graph across all views, representing each data point as a sparse linear combination of chosen landmarks. 
Unlike traditional methods that manage the anchor graph construction and the label propagation process separately, this paper proposes a unified optimization model that facilitates simultaneous learning of both. An effective alternating update algorithm with convergence guarantees is proposed to solve the unified optimization model.
Additionally, the method employs the obtained anchor graph and landmarks' low-dimensional representations to deduce low-dimensional representations for raw data. Following this,  a straightforward clustering approach is conducted on these low-dimensional representations to achieve the final clustering results. The effectiveness and efficiency of FSSMSC are validated through extensive experiments on multiple benchmark datasets of varying scales.}




\maketitle

\section{Introduction}

Real-world datasets commonly encompass multi-view features that depict objects from different perspectives. For instance, materials conveying similar meanings might be expressed in different languages. The field of machine learning has seen the emergence of a significant topic known as multi-view clustering. This approach seeks to partition data by exploring consensus across multiple views. A prominent branch within this context is multi-view subspace clustering, which has garnered substantial attention during the last decade due to its promising performance \cite{cao2015diversity,gao2015multi,wang2017exclusivity,brbic2018multi,luo2018consistent,zhang2018generalized,li2019reciprocal,li2019flexible,gao2020tensor,zhang2020tensorized,kang2020partition,zhang2020consensus,sun2021projective,si2022consistent,cai2023seeking}. These methodologies focus on acquiring self-representative affinity matrices by expressing each data point as a linear combination of all other data points in either the feature space or the latent space.

Despite the achievements of multi-view subspace clustering methods in various application domains, their application to large-scale datasets is hindered by significant time and memory overhead.
These methods involve two primary phases: graph construction and spectral clustering. In the graph construction phase, the computational complexity of creating the self-representative affinity matrix for unstructured data is $\mathcal{O}(n^3)$, where $n$ denotes the data size. The subsequent spectral clustering phase incurs a computational complexity of at least $\mathcal{O}(n^2)$ due to the use of Singular Value Decomposition (SVD). Furthermore, the overall space complexity is estimated at $\mathcal{O}(n^2)$. 
In recent years, several landmark-based methods, as illustrated by \cite{kang2020large,sun2021scalable,wang2021fast}, have emerged to enhance the efficiency of multi-view subspace clustering. These methods adopt an anchor graph strategy, generating a small number of landmarks (or anchors) and constructing anchor graphs to establish connections between the raw data and the landmarks. This anchor graph strategy significantly reduces computational and memory requirements, resulting in a linear complexity of $\mathcal{O}(n)$, where $n$ denotes the data size.

In many applications, a small amount of supervisory information is often available despite the resource-intensive nature of data annotation. And several landmark-based semi-supervised learning methods like \cite{liu2010large,wang2016scalable,zhang2022graph} have emerged to harness the available supervisory information. These methods encompass two core phases: anchor graph construction and label propagation. During anchor graph construction, an anchor graph is crafted by assessing similarities between raw data and landmarks. Subsequently, in label propagation, a graph structure is devised for landmarks with the previously derived anchor graph, confining the exploration of supervisory information to a much smaller graph and bypassing the complete graph construction for raw data. However, these methods handle anchor graph construction and label propagation in isolation. Consequently, the anchor graph's formation occurs independently of supervisory information, potentially yielding a sub-optimal anchor graph that impairs label propagation. Moreover, their applicability is restricted to single-view datasets, precluding the utilization of multi-view features.

To address the above limitations of existing landmark-based semi-supervised learning methods, we introduce a novel method termed Fast Scalable Semi-supervised Multi-view Subspace Clustering (FSSMSC). This method formulates a unified optimization model aimed at concurrently facilitating the processes of anchor graph construction and label propagation. Notably, FSSMSC extends the landmark-based multi-view subspace clustering methods to harness supervisory information. Furthermore, the proposed FSSMSC offers a computational and space complexity of linear order with respect to the data size, rendering it applicable to large-scale datasets with multi-view features. 


An efficient alternating update scheme is implemented to solve the proposed optimization model, cyclicly refreshing both the anchor graph and the landmarks' low-dimensional representations. This iterative procedure facilitates simultaneous advancement in both the anchor graph and the low-dimensional representations of the landmarks. Significantly, within this iterative procedure, the consensus anchor graph undergoes updates utilizing both the data features and the landmarks' representations, which contrasts with existing scalable semi-supervised learning methods such as \cite{liu2010large,wang2016scalable,zhang2022graph} that solely construct the anchor graph based on data features. The landmarks' representations act as pseudo supervision, facilitating the learning of the consensus anchor graph. 
While unsupervised clustering methods like \cite{wang2021fast} optimize anchor graph and landmarks simultaneously, they derive low-dimensional representations for clustering by employing Singular Value Decomposition (SVD) on the learned anchor graph, thereby handling anchor graph learning and low-dimensional representation learning separately. In contrast, our proposed semi-supervised learning method manages anchor graph construction and low-dimensional representation learning concurrently.

Our contributions can be summarized as follows:

\begin{itemize}

\item We introduce a method termed Fast Scalable Semi-supervised Multi-view Subspace Clustering (FSSMSC), which extends the existing landmark-based multi-view subspace clustering methods to harness supervisory information. Moreover, the proposed FSSMSC offers linear computational and space complexity, making it well-suited for large-scale datasets with multi-view features. 

\item Distinct from traditional methods that manage the anchor graph construction and the label propagation process separately, we formulate a unified optimization model that facilitates simultaneous learning of both.

\item An efficient alternating update scheme with convergence guarantees is introduced to solve the proposed optimization model, cyclicly updating both the anchor graph and the landmarks' low-dimensional representations.
Extensive experiments on multiple benchmark multi-view datasets of varying scales verify the efficiency and effectiveness of the proposed FSSMSC.
\end{itemize}

\textbf{Notations.} We denote matrices by boldface uppercase letters, e.g., $\mathbf{Z}$, vectors by boldface lowercase letters, e.g., $\mathbf{f}$, and scalars by lowercase letters, e.g., $\alpha$. We denote $z_{ij}$ or $\mathbf{Z}_{ij}$ as the $(i,j)$-th element of $\mathbf{Z}$. 
We use Tr($\mathbf{Z}$) to denote the trace of any square matrix $\mathbf{Z}$. $|\mathbf{Z}|$ is the matrix consisting of the absolute value of each element in $\mathbf{Z}$. We denote $\mbox{diag}(\mathbf{W})$ as a vector of diagonal elements in $\mathbf{W}$ and $\mbox{diag}(\mathbf{f})$ as a diagonal matrix with $\mathbf{f}$ being the diagonal elements. $\mathbf{W}^\top$ is the transpose of $\mathbf{W}$. We set $\mathbf{I}$ as the identity matrix and $\mathbf{0}$ as a matrix 
 or vector of all zeros. We give the notations of some norms, e.g., $\ell_1$-norm $\|\mathbf{Z}\|_1=\sum_{ij}|z_{ij}|$ and Frobenius norm (or $\ell_2$-norm of a vector) $\|\mathbf{Z}\|=\sqrt{\sum_{ij}z_{ij}^2}$. 

\section{Related Work}

\subsection{Subspace Clustering}
\label{sec: subspace clustering}
Subspace clustering assumes that the data points are drawn from a union of multiple low-dimensional subspaces, with each group corresponding to one subspace. This clustering paradigm finds applications in diverse domains, including motion segmentation \cite{elhamifar2013sparse}, face clustering \cite{zhang2019self}, and image processing \cite{yang2016ell}. Within subspace clustering, self-representative subspace clustering achieves state-of-the-art performance by representing each data point as a linear combination of other data points within the same subspace. Such a representation is not unique, and sparse subspace clustering (SSC) \cite{elhamifar2013sparse} pursues a sparse representation by solving the following problem
\begin{align}
\label{op:SSC}
\min_{\mathbf{Z}}~\|\mathbf{Z}\|_{1}+\gamma\|\mathbf{X-XZ}\|_1,
~\text{s.t.}~\mbox{diag}(\mathbf{Z})=\mathbf{0},
\end{align}
where $\mathbf{Z}$ is a $n\times n$ self-representative matrix for $n$ data points  $\mathbf{X}=[\mathbf{x}_1,\mathbf{x}_2,\ldots,\mathbf{x}_n]$, and $|\mathbf{Z}_{ij}|$ reflects the affinity between data point $\mathbf{x}_i$ and data point $\mathbf{x}_j$. Then the affinity matrix is induced by $\mathbf{Z}$ as $\frac{1}{2}\left(|\mathbf{Z}|+\left|\mathbf{Z}\right|^{\top}\right)$, which is further used for spectral clustering (SC) \cite{ng2002spectral} to obtain the final clustering results.

In this paper, we represent each data point using selected landmarks to facilitate linear computational and space complexity relative to the data size $n$.

\subsection{Semi-Supervised Multi-View Learning}

Owing to the prevalence of extensive real-world datasets with diverse features, numerous semi-supervised learning approaches have been developed specifically for multi-view data \cite{wang2009unified,cai2013heterogeneous,karasuyama2013multiple,nie2017auto,tao2017scalable,liang2020semi,qin2021semi,liang2022co}. 
The multiple graphs learning framework (AMGL) in \cite{nie2016parameter} automatically learned the optimal weights for view-specific graphs. Additionally, the work \cite{nie2017auto} presented a multi-view learning with adaptive neighbors (MLAN) framework, concurrently performing label propagation and graph structure learning.
\cite{liang2020semi} proposed a semi-supervised multi-view clustering method (GPSNMF) that constructs an affinity graph to preserve geometric information.
An anti-block-diagonal indicator matrix was introduced in \cite{qin2021semi} to enforce the block-diagonal structure of the shared affinity matrix. \cite{tang2021constrained} introduced a constrained tensor representation learning model to utilize prior weakly supervision for multi-view semi-supervised subspace clustering task.
The work \cite{liang2022co} proposed a novel semi-supervised multi-view clustering approach (CONMF) based on orthogonal non-negative matrix factorization.
Despite their success in semi-supervised learning, these methods typically entail affinity matrix construction and exhibit at least $\mathcal{O}(n^2)$
 computational complexity, where 
$n$ denotes the data size. Notably, a semi-supervised learning approach based on adaptive regression (MVAR) was introduced in \cite{tao2017scalable}, offering linear scalability with $n$ and quadratic scalability with feature dimension.

This paper introduces FSSMSC, a fast and scalable semi-supervised multi-view learning method to address the prevalent high computational complexity in existing approaches. FSSMSC exhibits linear computational and space complexity relative to data size and feature dimension.

\section{Methodology}

\begin{table}[!t]
    \centering
    \caption{Main notations}
    \label{tab:Main notations}
    \begin{tabular}{@{}l|l@{}} 
        \toprule
        Notations & Descriptions \\
        \midrule
        $n$ & Number of data points within the whole dataset \\
        \midrule
        $m$ & Number of selected landmarks \\
        \midrule
        $k$ & Number of clusters (categories) \\
        \midrule
        $n_\ell$ & Number of chosen labeled samples \\
        \midrule
        $\mathbf{X}^{(v)} \in \mathbb{R}^{d_v\times n}$ & Data matrix for the $v$-th view \\
        \midrule
        $\mathbf{U}^{(v)}\in \mathbb{R}^{d_v\times m}$ & Data matrix of the selected $m$ landmarks for the $v$-th view \\
        \midrule
        $\mathbf{M},\mathbf{C}\in \mathbb{R}^{n_\ell \times n_\ell}$ & Matrices of pairwise constraints\\
        \midrule
        $\mathbf{L}_\mathbf{M},\mathbf{L}_\mathbf{C} \in \mathbb{R}^{n_\ell \times n_\ell}$ & Laplacian matrices of $\mathbf{M}$ and $\mathbf{C}$ \\
        \midrule
        $\mathbf{Z} \in \mathbb{R}^{m\times n}$ & Consensus anchor graph \\
        \midrule
        $\mathbf{A}\in \mathbb{R}^{k\times m}$ & Low-dimensional representations of the selected $m$ landmarks \\
        \bottomrule
    \end{tabular}
\end{table}

This section introduces the proposed Fast Scalable Semi-supervised Multi-view Subspace Clustering (FSSMSC) approach. We present the main notations of this paper in \cref{tab:Main notations}.  

\subsection{The Unified Optimization Model}

For multi-view data, let $V$ denote the number of views, and $\mathbf{X}^{(v)}\in \mathbb{R}^{d_v\times n}$ represent the data matrix for $n$ data points $\mathbf{x}_1,\mathbf{x}_2,\ldots,\mathbf{x}_n$ in the $v$-th view, where $d_v$ and $n$ stand for the dimension of the $v$-th view and the data size, respectively. We perform $k$-means clustering on the concatenated features of all views to establish consistent landmarks across all views. The resulting clustering centers are then partitioned into distinct views, denoted as $\mathbf{U}^{(v)}\in \mathbb{R}^{d_v\times m}$, $v=1,\dots, V$. 
Assuming the first $n_\ell$ data points are provided with labels, while the remaining data points remain unlabeled. The label assigned to $\mathbf{x}_i$ is denoted as $\ell_i\in\{1,\ldots,k\}$, $i=1,\ldots,n_\ell$, where $k$ is the number of data categories. We formulate a more flexible form of supervisory information for the $n_\ell$ labeled data points, i.e., pairwise constraints $\mathcal{M} = \{(i,j)|\ell_i=\ell_j\}$ and $\mathcal{C} = \{(i,j)|\ell_i\neq \ell_j\}$, where $\mathcal{M}$ and $\mathcal{C}$ denote the sets of data pairs subject to must-link and cannot-link constraints, respectively.

Scalable semi-supervised learning methods like \cite{liu2010large,wang2016scalable,zhang2022graph} compute the soft label vectors for $n$ data points by a linear combination of the soft label vectors of $m$ chosen landmarks (anchors), where $m$ is notably smaller than $n$.
Similarly, we compute low-dimensional representations $\mathbf{F}\in\mathbb{R}^{k\times n}$ for $n$ data points through a linear combination of low-dimensional representations $\mathbf{A} \in \mathbb{R}^{k\times m}$ of $m$ landmarks 
\begin{align}
\label{eq: FArelation}
\mathbf{f}_i = \sum \nolimits_{j=1}^m \mathbf{Z}_{ji}\mathbf{a}_j,~~~i=1,\ldots,n,
\end{align}
where $\mathbf{f}_i$ is the $i$-th column of $\mathbf{F}$, $\mathbf{a}_j$ is the $j$-th column of $\mathbf{A}$. Here, $k$ signifies the dimension of the low-dimensional space, consistently set as the number of data categories throughout this paper. Additionally, $\mathbf{Z} = (\mathbf{Z}_{ji})\in \mathbb{R}^{m\times n}$ is a consensus anchor graph capturing the relationships (similarities) between $n$ data points and $m$ selected landmarks.

However, these previous methods in \cite{liu2010large,wang2016scalable,zhang2022graph} treat anchor graph learning and the learning of landmarks' soft label vectors as disjoint sub-problems. In contrast, we formulate a unified optimization model facilitating concurrent advancement in both the anchor graph $\mathbf{Z}$ and the low-dimensional representations $\mathbf{A}$ of the landmarks.
Specifically, the proposed unified optimization model of the anchor graph $\mathbf{Z}$ and low-dimensional representations $\mathbf{A}$ is based on the following observations.

Firstly, each entry $\mathbf{Z}_{ji}$ within the anchor graph $\mathbf{Z}$ should capture the affinity between the $i$-th data point and the $j$-th landmark. 
To fulfill this requirement, we extend the optimization model in \eqref{op:SSC} to a landmark-based multi-view clustering scenario. 
As discussed in \cref{sec: subspace clustering}, the optimization model \eqref{op:SSC} proposed in \cite{elhamifar2013sparse} represents each data point as a sparse linear combination of other data points within the same subspace, and induces the affinity matrix using these linear combination coefficients.
Denote $\mathbf{X}^{(v)}=[\mathbf{x}^{(v)}_1,\dots,\mathbf{x}^{(v)}_n]$, where $\mathbf{x}^{(v)}_i$ represents the feature vector of the $i$-th data point $\mathbf{x}_i$ in the $v$-th view.
In this paper, we depict each $\mathbf{x}_i^{(v)}$ as a linear combination of the landmarks $\mathbf{U}^{(v)}$ and assume that the combination coefficients are shared across views. Therefore, we explore the relationship among views by assuming that the anchor graph $\mathbf{Z}$ is shared across views to capture their consensus. Furthermore, we enforce coefficient sparsity by utilizing the $\ell_1$-norm on $\mathbf{Z}$, leading to the following optimization problem:
\begin{align}
\label{eq: construct Z}
\min_{\mathbf{Z} \geq 0} ~&\varphi(\mathbf{Z}) = \frac{1}{2}\sum_{v=1}^V\|\mathbf{X}^{(v)}-\mathbf{U}^{(v)}\mathbf{Z}\|^2 + \lambda_Z \|\mathbf{Z}\|_1, 
\end{align}
where we impose non-negativity constraints on the elements in $\mathbf{Z}$. The parameter $\lambda_Z>0$ governs the sparsity level in $\mathbf{Z}$.

 Secondly, the low-dimensional representations for the $n$ data points, denoted as $\mathbf{F=AZ}$ following (\ref{eq: FArelation}), need to conform to the provided pairwise constraints. Specifically, data points with must-link (cannot-link) constraints should exhibit similar (distinct) low-dimensional representations.
To fulfill this requirement, we encode the pairwise constraints into two matrices, $\mathbf{M}=(m_{ij})\in\mathbb{R}^{n_\ell \times n_\ell}$ and $\mathbf{C}=(c_{ij})\in\mathbb{R}^{n_\ell \times n_\ell}$, with elements defined as:
\begin{align}
     m_{ij} = \begin{cases} \frac{1}{n_m},   &\text{if} ~(i,j)\in \mathcal{M}\\
      0, &\text{otherwise}
     \end{cases},~ c_{ij} = \begin{cases} \frac{1}{n_c},   &\text{if} ~(i,j)\in \mathcal{C}\\
      0, &\text{otherwise}
     \end{cases}
 \end{align}
 where $n_m$ ($n_c$) is the total number of must-link (cannot-link) constraints in $\mathcal{M}$ ($\mathcal{C}$).
 The matrix $\mathbf{M}$ represents prior pairwise similarity constraints among data points. The element $m_{ij}$ functions as a reward factor, with a positive value indicating a constraint that data points $i$ and $j$ should be grouped into the same cluster. In contrast, the matrix $\mathbf{C}$ denotes prior pairwise dissimilarity constraints. The element $c_{ij}$ acts as a penalty factor, where positive values suggest that data points $i$ and $j$ should be assigned to different clusters. To balance the influence of must-link and cannot-link constraints, we normalize the non-zero entries in $\mathbf{M}$ and $\mathbf{C}$ by $n_m$ and $n_c$, respectively.
Subsequently, we utilize $m_{ij}$ and $c_{ij}$ to measure intra-class and inter-class distances for labeled data points, respectively, in the low-dimensional space.
Specifically, we ensure that data pairs in $\mathcal{M}$ share similar low-dimensional representations through the minimization of
\begin{align}
 \phi_\mathcal{M}(\mathbf{A},\mathbf{Z})&=\sum \nolimits_{i,j=1}^{n_\ell} \frac{m_{ij}}{2} \|\mathbf{A}\mathbf{z}_i - \mathbf{A}\mathbf{z}_j\|^2  \\ &=\operatorname{Tr}(\mathbf{AZ_\ell} \mathbf{L}_\mathbf{M}(\mathbf{AZ_\ell})^\top) , \nonumber
 \end{align}
where $\mathbf{Z}_\ell$ is the first $n_\ell$ columns of $\mathbf{Z}$ and $\mathbf{z}_i$ is any $i$-th column of $\mathbf{Z}$, $\mathbf{L}_\mathbf{M}=\mathbf{D_M}-\mathbf{M}$, where $\mathbf{D_M}$ is a diagonal degree matrix with the $i$-diagonal element being $\sum_{j=1}^{n_\ell} m_{ij}$. Similarly, we ensure that data pairs in $\mathcal{C}$ exhibit distinct low-dimensional representations through the maximization of 
\begin{align}
 \phi_\mathcal{C}(\mathbf{A},\mathbf{Z}) &=\sum \nolimits_{i,j=1}^{n_\ell} \frac{c_{ij}}{2} \|\mathbf{A}\mathbf{z}_i - \mathbf{A}\mathbf{z}_j\|^2  \\
 &= \operatorname{Tr}(\mathbf{AZ_\ell} \mathbf{L}_\mathbf{C}(\mathbf{AZ_\ell})^\top), \nonumber
 \end{align}
where $\mathbf{L}_\mathbf{C}=\mathbf{D_C}-\mathbf{C}$ and $\mathbf{D_C}$ is a diagonal degree matrix with the $i$-diagonal element being $\sum_{j=1}^{n_\ell} c_{ij}$.

   

\begin{algorithm}[tb]
   \caption{FSSMSC}
   \label{alg: FSSMSC}
\begin{algorithmic}
   \State {\bfseries Input:} Data features $\mathbf{X}^{(v)},v=1,...,V$, landmarks $\mathbf{U}^{(v)},v=1,...,V$, pairwise constraints $\mathbf{L_M},\mathbf{L_C}$, 
 and parameters $\lambda_Z,\beta,\lambda_M$.
   \State {\bfseries Initialize:}
   $\mathbf{B}=(\lambda \mathbf{I}+\sum_{v=1}^V\mathbf{U}^{(v)^\top}\mathbf{U}^{(v)})^{-1}$ $\cdot$ \\ $(\sum_{v=1}^V\mathbf{U}^{(v)^\top}\mathbf{X}^{(v)})$, $\mathbf{Z}=\min(m_Z,\max(\mathbf{0},\mathbf{B}-\frac{\lambda_Z}{\lambda}))$, $\mathbf{q=ZZ^\top1}$, $\mathbf{A=0}$, and $\Lambda=\mathbf{0}$.
   \Repeat
   \State Update $\mathbf{A}$ by solving the trace ratio problem (\ref{eq:solve A2}); 
   \State Update $\mathbf{Z}$ by the formula (\ref{eq: update Z});
    \State Update $\mathbf{q}$ by the formula (\ref{eq: update q});
    \State Update $\mathbf{B}$ by the least-square solution (\ref{eq: update B});
    \State Update the multiplier $\Lambda$ by (\ref{eq: update multipliers});
   \Until{converged}
   \State {\bfseries Output:} $\mathbf{A,Z,B}$.
\end{algorithmic}
\end{algorithm}

 Thirdly, the representations $\mathbf{A}$ of landmarks should exhibit smoothness on the graph structure of landmarks, signifying that landmarks with higher similarity ought to possess spatially close representations. This objective is commonly achieved through the application of manifold smoothing regularization. 
Specifically, for a symmetric similarity matrix $\mathbf{W}\in \mathbb{R}^{m\times m}$ of $m$ chosen landmarks, the manifold smoothing regularization is given by 
\begin{align}
\label{eq:manifold smoothing}
\sum_{i,j=1}^m \frac{\mathbf{W}_{ij}}{2}\|\frac{\mathbf{a}_i}{\sqrt{\mathbf{D}_{ii}}}-\frac{\mathbf{a}_j}{\sqrt{\mathbf{D}_{jj}}}\|^2 
=\text{Tr}\left(\mathbf{A}(\mathbf{I}-\mathbf{D}^{-\frac{1}{2}}\mathbf{W}\mathbf{D}^{-\frac{1}{2}})\mathbf{A}^\top\right),
\end{align}
where $\mathbf{D}$ is the diagonal degree matrix with $\mathbf{D}_{ii} = \sum_{j=1}^m \mathbf{W}_{ij}$. Minimizing this objective ensures that landmarks $i$ and $j$ with high similarity value $\mathbf{W}_{ij}$ have their scaled low-dimensional representations $\mathbf{a}_i$ and $\mathbf{a}_j$ close to each other.  

 Given that $\mathbf{Z}_{ij}\geq 0$ signifies the similarity between landmark $i$ and data point $j$, we define the graph structure $\mathbf{W}$ for landmarks as $\mathbf{W} = \mathbf{Z}\mathbf{Z}^\top$, consistent with the definition in prior literature \cite{wang2016scalable,zhang2022graph}. We subsequently derive the normalized Laplacian matrix $\mathbf{L}_{\mathbf{W}}$ using $\mathbf{L}_{\mathbf{W}} = \mathbf{I}-\mathbf{D}^{-\frac{1}{2}}\mathbf{W}\mathbf{D}^{-\frac{1}{2}}$. Then, we introduce the manifold smoothing regularization as follows:
\begin{align}
\label{eq: smooth term}
\phi_s(\mathbf{A},\mathbf{Z})=&\text{Tr}\left(\mathbf{A}\mathbf{L_W}\mathbf{A}^\top\right) \nonumber \\
= & \text{Tr}\left(\mathbf{A}(\mathbf{I}-\mathbf{D_Z}^{-1/2}\mathbf{Z}\mathbf{Z}^\top \mathbf{D_Z}^{-1/2})\mathbf{A}^\top\right),
\end{align}
where $\mathbf{D_Z}=\text{diag}(\mathbf{Z}\mathbf{Z}^\top\mathbf{1}_m)$ and $\mathbf{1}_m=(1,\cdots,1)^\top\in\mathbb{R}^m$.


By consolidating the aforementioned optimization requirements, we establish a unified optimization problem for the simultaneous learning of $\mathbf{Z}$ and $\mathbf{A}$ as follows:
\begin{align}
\label{op: FSSMSC}
\min_{\mathbf{A},~\mathbf{Z}}~  &\varphi(\mathbf{Z}) + \beta \cdot \frac{\phi_s(\mathbf{A},\mathbf{Z})+\lambda_M\phi_\mathcal{M}(\mathbf{A},\mathbf{Z})}{\phi_\mathcal{C}(\mathbf{A},\mathbf{Z})} \\
 & \text{s.t.}~~\mathbf{A}\mathbf{A}^\top=\mathbf{I}, ~\mathbf{Z} \geq 0, \nonumber
\end{align}
where $\beta,\lambda_M>0$ are penalty parameters.

We introduce auxiliary variable $\mathbf{q}\in \mathbb{R}^m$ subject to the constraint $\mathbf{q}=\mathbf{Z}\mathbf{Z}^\top\mathbf{1}_m$. Subsequently, we replace $\mathbf{D}_\mathbf{Z}$ with $\mathbf{D_q}=\text{diag}(\mathbf{q})$ in $\phi_s(\mathbf{A},\mathbf{Z})$, denoted as $\tilde{\phi}_s(\mathbf{A},\mathbf{Z},\mathbf{q})$. To facilitate numerical stability and convergence analysis, we impose an upper bound $m_Z>0$ on the elements in $\mathbf{Z}$ and enforce upper bound $C_q>0$ and lower bound $\epsilon_q>0$ on the elements in $\mathbf{q}$. Additionally, a positive constant $\epsilon >0$ is added to the denominator in (\ref{op: FSSMSC}). Consequently, the optimization problem is formulated as follows:
\begin{align}
\label{op: FSSMSC new}
\min_{\mathbf{A},\mathbf{Z},\mathbf{q}}~  &\varphi(\mathbf{Z}) + \beta \cdot \frac{\tilde{\phi}_s(\mathbf{A},\mathbf{Z},\mathbf{q})+\lambda_M\phi_\mathcal{M}(\mathbf{A},\mathbf{Z})}{\phi_\mathcal{C}(\mathbf{A},\mathbf{Z})+\epsilon} \\
&+\frac{\lambda_0}{2}\|\mathbf{q}-\mathbf{Z}\mathbf{Z}^\top\mathbf{1}_m\|^2 \nonumber \\
& \text{s.t.}~~\mathbf{A}\mathbf{A}^\top=\mathbf{I}, ~m_Z \geq \mathbf{Z} \geq 0, C_q \geq \mathbf{q} \geq \epsilon_q,  \nonumber
\end{align}
where $\frac{\lambda_0}{2}\|\mathbf{q}-\mathbf{Z}\mathbf{Z}^\top\mathbf{1}_m\|^2$ is the penalized term for the constraint $\mathbf{q}=\mathbf{Z}\mathbf{Z}^\top\mathbf{1}_m$ and $\lambda_0>0$ is a penalized parameter.

\subsection{Numerical Algorithm}
We employ the Alternating Direction Method of Multipliers (ADMM) approach to solve the optimization problem (\ref{op: FSSMSC new}).

 We introduce auxiliary variables $\mathbf{B}=\mathbf{Z}$, leading to the augmented Lagrangian expression as follows:
\begin{align}
\label{eq:L}
&\mathcal{L}_\lambda(\mathbf{A},\mathbf{Z},\mathbf{q},\mathbf{B},\Lambda) \\
 = &\frac{1}{2}\sum_{v=1}^V\|\mathbf{X}^{(v)}-\mathbf{U}^{(v)}\mathbf{B}\|^2 + \lambda_Z \|\mathbf{Z}\|_1 + (\Lambda,\mathbf{B-Z})\nonumber\\
    & + \beta \cdot \frac{\text{Tr}\left(\mathbf{A}(\mathbf{I}-\mathbf{D_q}^{-1/2}\mathbf{Z}\mathbf{Z}^\top \mathbf{D_q}^{-1/2})\mathbf{A}^\top\right)}{\text{Tr}\left(\mathbf{A}\mathbf{Z_\ell}\mathbf{L}_\mathbf{C}(\mathbf{A}\mathbf{Z_\ell})^\top\right)+\epsilon} \nonumber\\
    & + \beta \cdot \frac{\lambda_M  \text{Tr}\left(\mathbf{A}\mathbf{Z_\ell}\mathbf{L}_\mathbf{M}(\mathbf{A}\mathbf{Z_\ell})^\top \right)}{\text{Tr}\left(\mathbf{A}\mathbf{Z_\ell}\mathbf{L}_\mathbf{C}(\mathbf{A}\mathbf{Z_\ell})^\top\right)+\epsilon}  + \frac{\lambda}{2}\|\mathbf{B-Z}\|^2 \nonumber \\
    &+ \frac{\lambda_0}{2}\|\mathbf{q}-\mathbf{Z}\mathbf{Z}^\top\mathbf{1}_m\|^2 + \ell_{S_0}(\mathbf{A}) + \ell_{S_1}(\mathbf{Z})+\ell_{S_2}(\mathbf{q}),\nonumber
\end{align}
where $\Lambda$ is the Lagrange multiplier. The trade-off parameters $\lambda_M$, $\lambda_Z$, and $\beta$ are utilized to harmonize various components.
We define $S_0=\{\mathbf{A}\in \mathbb{R}^{k\times m}|\mathbf{A} \mathbf{A}^\top=\mathbf{I}\}$, $S_1 = \{\mathbf{Z}\in \mathbb{R}^{m\times n}|m_Z\geq \mathbf{Z}\geq 0\}$ and $S_2=\{\mathbf{q}\in \mathbb{R}^m|\epsilon_q \leq \mathbf{q} \leq C_q\}$, with their respective indicator functions denoted as $\ell_{S_0}(\cdot)$, $\ell_{S_1}(\cdot)$ and $\ell_{S_2}(\cdot)$. 
Throughout our experiments, we fix $\lambda$ and $\lambda_0$ as $100$, $m_Z$ as $1$, $C_q$ as $n$, and $\epsilon$ and $\epsilon_q$ as $1\times 10^{-5}$.

In each iteration, we alternatingly minimize over $\mathbf{A}$, $\mathbf{Z}$, $\mathbf{q}$, $\textbf{B}$ and then take gradient ascent steps on the Lagrange multiplier. The minimizations over $\mathbf{A,Z,q,B}$ hold the other variables fixed and use the most recent value of each variable. 

$\bullet$ Update of $\mathbf{A}$. The sub-problem on optimizing $\mathbf{A}$ can be written as an equivalent form
\begin{align}
\label{eq:solve A2}
&\max_{\mathbf{AA}^\top=\mathbf{I}} ~  \frac{\text{Tr}\left(\mathbf{A}\mathbf{L}_d\mathbf{A}^\top \right)} {\text{Tr}\left(\mathbf{A}\mathbf{L}_f\mathbf{A}^\top \right)}, 
\end{align}
where 
$
\mathbf{L}_f=\lambda_M\left(\mathbf{Z}_\ell\mathbf{L}_\mathbf{M}\mathbf{Z}_\ell^\top \right)+ \mathbf{I}-\mathbf{D}_\mathbf{q}^{-1/2}\mathbf{Z}\mathbf{Z}^\top\mathbf{D}_\mathbf{q}^{-1/2}$
and $\mathbf{L}_d=\mathbf{Z}_\ell\mathbf{L}_\mathbf{C}\mathbf{Z}_\ell^\top+\frac{\epsilon}{k}\cdot \mathbf{I}$. The trace ratio optimization problem (\ref{eq:solve A2}) can be effectively addressed using ALGORITHM 4.1 presented in \cite{ngo2012trace}. 

$\bullet$ Update of $\mathbf{Z}$. The sub-problem on optimizing $\mathbf{Z}$ can be written as follows:
\begin{align}
\label{eq:P}
&\min_{\mathbf{Z}\in \mathcal{S}_1} P(\mathbf{Z}):= \beta \cdot \frac{\tilde{\phi}_s(\mathbf{A},\mathbf{Z},\mathbf{q})+\lambda_M\phi_\mathcal{M}(\mathbf{A},\mathbf{Z})}{\phi_\mathcal{C}(\mathbf{A},\mathbf{Z})+\epsilon} - (\Lambda,\mathbf{Z}) \nonumber\\
& +\lambda_Z\sum_{i,j} \mathbf{Z}_{ij} + \frac{\lambda}{2}\|\mathbf{B-Z}\|^2 + \frac{\lambda_0}{2}\|\mathbf{q}-\mathbf{Z}\mathbf{Z}^\top\mathbf{1}_m\|^2.
\end{align}
We employ a one-step projected gradient descent strategy to update $\mathbf{Z}$. By deriving the gradient of $P(\mathbf{Z})$ with respect to $\mathbf{Z}$, we can obtain: 
\begin{align}
\nabla_\mathbf{Z}P(\mathbf{Z})  = &-\Lambda + \lambda(\mathbf{Z-B}) + \lambda_Z \mathbf{1}_{m\times n} \\
& + [\beta \mathbf{A}^\top\mathbf{A}\mathbf{Z}_\ell \mathbf{L}_{pc},\mathbf{0}_{m\times (n-n_\ell)}] -\frac{2\beta}{\tau_2}\mathbf{D}_\mathbf{q}^{-1/2}\mathbf{A}^\top\mathbf{A}\mathbf{D}_\mathbf{q}^{-1/2}\mathbf{Z}  \nonumber \\
&+ \lambda_0 \left(\left(\mathbf{ZZ}^\top\mathbf{1}_m-\mathbf{q}\right)\mathbf{1}_m^\top\mathbf{Z}+\mathbf{1}_m\left(\mathbf{ZZ}^\top\mathbf{1}_m-\mathbf{q}\right)^\top\mathbf{Z}\right), \nonumber
\end{align}
where $\mathbf{L}_{pc}=\frac{2\lambda_M}{\tau_2}\mathbf{L}_\mathbf{M}-\frac{2\tau_1}{\tau_2^2}\mathbf{L}_\mathbf{C}$, $\tau_2=\phi_\mathcal{C}(\mathbf{A},\mathbf{Z})+\epsilon$, and
$\tau_1=\tilde{\phi}_s(\mathbf{A},\mathbf{Z},\mathbf{q})+\lambda_M\phi_\mathcal{M}(\mathbf{A},\mathbf{Z})$.
Then we update $\mathbf{Z}$ by 
\begin{align}
\label{eq: update Z}
\mathbf{Z} = \min(m_Z,\max(\mathbf{0},\mathbf{Z}-\eta_z \nabla_\mathbf{Z}P(\mathbf{Z}))),
\end{align}
where $\eta_z$ is the step size, consistently set as $\frac{1}{\lambda}$ in this paper.

$\bullet$ Update of $\mathbf{q}$. 
The sub-problem on optimizing $\mathbf{q}$ can be written as follows:
\begin{align}
\label{eq:Q}
\hspace{-3mm}\min_{\mathbf{q}\in \mathcal{S}_2} Q(\mathbf{q}):=&  \beta \cdot \frac{\tilde{\phi}_s(\mathbf{A},\mathbf{Z},\mathbf{q})}{\phi_\mathcal{C}(\mathbf{A},\mathbf{Z})+\epsilon} + \frac{\lambda_0}{2}\|\mathbf{q}-\mathbf{Z}\mathbf{Z}^\top\mathbf{1}_m\|^2.
\end{align} 
By deriving the gradient of $Q(\mathbf{q})$ with respect to $\mathbf{q}$, we can obtain:
\begin{align}
\nabla_\mathbf{q}Q(\mathbf{q})  = &\frac{\beta}{\tau_2}\text{diag}\left(\mathbf{ZZ}^\top \text{diag}(\mathbf{q}^{-1/2})\mathbf{A}^\top\mathbf{A} \text{diag}(\mathbf{q}^{-3/2})\right) \\
&  + \lambda_0 (\mathbf{q}-\mathbf{ZZ}^\top\mathbf{1}_m). \nonumber
\end{align}
Then we update $\mathbf{q}$ by
\begin{align}
\label{eq: update q}
\mathbf{q} = \min(C_q,\max(\epsilon_q,\mathbf{q}-\eta_q \nabla_\mathbf{q}Q(\mathbf{q}))),
\end{align}
where $\eta_q$ is the step size, consistently set as $\frac{1}{\lambda}$ in this paper.

$\bullet$ Update of $\mathbf{B}$. The sub-problem on optimizing $\mathbf{B}$ is 
\begin{align}
\min_{\mathbf{B}}\frac{1}{2}\sum_{v=1}^V\|\mathbf{X}^{(v)}-\mathbf{U}^{(v)}\mathbf{B}\|^2 +(\Lambda,\mathbf{B-Z})+ \frac{\lambda}{2}\|\mathbf{B-Z}\|^2. 
\end{align}
The solution to the above least-square problem is
\begin{align}
\label{eq: update B}
\mathbf{B}=(\sum_{v=1}^V\mathbf{U}^{(v)^\top}\mathbf{U}^{(v)}+\lambda \mathbf{I})^{-1}(\sum_{v=1}^V\mathbf{U}^{(v)^\top}\mathbf{X}^{(v)} + \lambda \mathbf{Z} - \Lambda).
\end{align}
Since $\mathbf{U}^{(v)}\in \mathbb{R}^{d_v\times m}$, $\mathbf{X}^{(v)}\in \mathbb{R}^{d_v\times n}$, the computational complexity of solving the least-square problem amounts to $\mathcal{O}(m^3+m^2n+\sum_{v=1}^Vd_vmn)$.

$\bullet$ Update of the Lagrange multiplier $\Lambda$.
\begin{align}
\label{eq: update multipliers}
\Lambda &= \Lambda + \lambda (\mathbf{B-Z}). 
\end{align}
The detailed procedure is summarized in \cref{alg: FSSMSC}.

\subsection{Clustering Label Inference}

Upon obtaining the anchor graph $\mathbf{Z}$ and landmarks' low-dimensional representations $\mathbf{A}$ from Algorithm \ref{alg: FSSMSC}, the low-dimensional representations for the $n$ data points are computed as $\mathbf{F=AZ}=[\mathbf{f}_1,\mathbf{f}_2,\ldots,\mathbf{f}_n]$. Subsequently, the cluster label $c_i$ for each data point $\mathbf{x}_i,i=1,2,\ldots,n$, is inferred using 
\begin{align}
c_i = \ell_{idx}, ~~~idx = \mathop{\arg \min} \limits_{j\in\{1,2,\ldots,n_\ell\}} \|\mathbf{f}_i-\mathbf{f}_j\|.
\end{align}

\subsection{Complexity Analysis}

We begin by conducting a computational complexity analysis of the proposed FSSMSC. For updating $\mathbf{Z}$, the complexity is $\mathcal{O}(m^2n+mn_\ell^2)$. The update of $\mathbf{q}$ involves a complexity of $\mathcal{O}(m^2n)$. The complexity of updating $\mathbf{A}$ is $\mathcal{O}(m^2n+mn_\ell^2+m^3)$, while updating $\mathbf{B}$ requires $\mathcal{O}(dmn+m^3+m^2n)$, where $d$ stands for the total dimension of all features. The updates for $\Lambda$ incur a complexity of $\mathcal{O}(mn)$. Therefore, the overall computational complexity of FSSMSC amounts to $\mathcal{O}(m^2n+mn_\ell^2+m^3+dmn)$.
Moving on to the space complexity evaluation, considering matrices such as $\mathbf{X}^{(v)}\in \mathbb{R}^{d_v\times n}$, $\mathbf{U}^{(v)}\in \mathbb{R}^{d_v\times m}$ for $v=1,...,V$, $\mathbf{L_M},\mathbf{L_C}\in \mathbb{R}^{n_\ell\times n_\ell}$, $\mathbf{Z,B},\Lambda \in \mathbb{R}^{m\times n}$, and $\mathbf{A}\in \mathbb{R}^{k\times m}$, the aggregate space complexity totals $\mathcal{O}(dn+n_\ell^2+mn+kn)$.

In summary, the proposed method demonstrates linear complexity, both in computation and space, with respect to the data size $n$.

\subsection{Convergence Analysis}

We denote $\{\mathbf{A}^{k},\mathbf{Z}^{k},\mathbf{q}^k,\mathbf{B}^k,\Lambda^k\}$ as the variables generated by Algorithm \ref{alg: FSSMSC} at iteration $k$, and establish convergence guarantees for Algorithm \ref{alg: FSSMSC}.

\begin{lemma}
\label{lemma: main lemma main}
Let $h(\mathbf{B})=\frac{1}{2}\sum_{v=1}^V\|\mathbf{X}^{(v)}-\mathbf{U}^{(v)}\mathbf{B}\|^2$ and $P,Q$ are given in \eqref{eq:P} and \eqref{eq:Q}. Then  $\nabla_\mathbf{B} h(\mathbf{B})$ is Lipschitz continuous with a Lipschitz constant $L_h>0$. $\nabla_\mathbf{Z} P(\mathbf{Z})$ is Lipschitz continuous in $\mathcal{S}_1$ with a Lipschitz constant $\hat{L}_Z>0$. $\nabla_\mathbf{q} Q(\mathbf{q})$ is Lipschitz continuous in $\mathcal{S}_2$ with a Lipschitz constant $\hat{L}_q>0$.
\end{lemma}

\begin{theorem} 
\label{theorem: main theorem main}
If $\lambda> L_h+2+2L_h^2$, $\eta_q<\frac{1}{\hat{L}_q}$, and $\eta_z<\frac{1}{\hat{L}_Z}$, the following properties hold:

(1) $\mathcal{L}_\lambda(\mathbf{A}^{k},\mathbf{Z}^{k},\mathbf{q}^k,\mathbf{B}^k,\Lambda^k)$ in \eqref{eq:L} is lower bounded;

(2) There exists a constant $C>0$ such that
\begin{align}
&\mathcal{L}_\lambda(\mathbf{A}^{k},\mathbf{Z}^{k},\mathbf{q}^k,\mathbf{B}^k,\Lambda^k) 
- \mathcal{L}_\lambda(\mathbf{A}^{k+1},\mathbf{Z}^{k+1},\mathbf{q}^{k+1},\mathbf{B}^{k+1},\Lambda^{k+1}) \nonumber \\
\geq &C \left(\|\mathbf{B}^{k+1}-\mathbf{B}^k\|^2+ \|\mathbf{q}^{k+1}-\mathbf{q}^{k}\|^2\right), \forall k\in \mathbb{N};
\end{align}

(3) When $k\rightarrow \infty$, we have
\begin{align}
&\|\mathbf{B}^{k+1}-\mathbf{B}^k\|\rightarrow 0,~~\|\mathbf{Z}^{k+1}-\mathbf{Z}^k\|\rightarrow 0, \nonumber \\ &\|\Lambda^{k+1}-\Lambda^k\|\rightarrow 0,~~\|\mathbf{q}^{k+1}-\mathbf{q}^{k}\| \rightarrow 0. 
\end{align}
\end{theorem}
The proofs for the aforementioned lemma and theorem are available in the appendix.

\begin{table}[!ht]
\centering
\caption{Statistics of the datasets}
\label{tab: multi-view datasets}
\begin{tabular}{@{}lcccc@{}} 
\toprule
Dataset & $\#$Samples & $\#$Views & $\#$Categories \\
\midrule
Caltech7 & 1,474 & 6 & 7 \\
HW & 2,000 & 6 & 10 \\
Caltech20 & 2,386 & 6 & 20 \\
Reuters & 18,758 & 5 & 6 \\
NUS & 30,000 & 5 & 31 \\
Caltech101 & 9,144 & 6 & 102 \\
YoutubeFace$\_$sel & 101,499 & 5 & 31 \\
\bottomrule
\end{tabular}
\end{table}

\section{Experiments}

\begin{table}[!ht]
\centering
\footnotesize
\caption{Comparison of clustering performance and running time. The best results are highlighted in boldface, and the second-best results are underlined. OM means Out-of-Memory.}
\label{tab: multi-view results}
\setlength{\tabcolsep}{1.14mm}{
\begin{tabular*}{\textwidth}{@{\extracolsep\fill}lcccccccc}
\toprule
Dataset  & FPMVS  & LMVSC & MVAR  &  MLAN & AMGL & GPSNMF & CONMF & FSSMSC\\
\midrule
\multicolumn{9}{c}{\textbf{ACC($\%$)}} \\
\midrule
Caltech7 &  57.1 & 70.6 & 91.7 $\pm$ 1.7 & 91.6 $\pm$ 2.4 & 85.3 $\pm$ 1.6 & \underline{93.1} $\pm$ 1.5  & 92.7 $\pm$ 1.6 & \textbf{93.7} $\pm$ 1.9 \\
HW  & 82.3 & 89.4 & 95.6 $\pm$ 2.0 & \textbf{97.3} $\pm$ 0.2 & 86.1 $\pm$ 2.3 & \underline{97.1} $\pm$ 0.8 & 97.0 $\pm$ 0.5 & \underline{97.1} $\pm$ 0.5 \\
Caltech20  & 66.4 & 48.8 & 78.9 $\pm$ 1.8 & 78.1 $\pm$ 2.2 & 68.4 $\pm$ 3.4 & \underline{80.8} $\pm$ 1.4 & 80.0 $\pm$ 1.2 & \textbf{82.4} $\pm$ 1.8 \\
Reuters  & 57.5 & 49.6 & OM & OM & 63.1 $\pm$ 1.7 & 73.1 $\pm$ 1.3 & \underline{75.0} $\pm$ 1.0 & \textbf{75.8} $\pm$ 2.4\\
NUS & 19.4 & 14.9 & 25.5 $\pm$ 1.8 & OM & 19.1 $\pm$ 1.3 & \underline{28.3} $\pm$ 1.0 & 24.8 $\pm$ 1.0 & \textbf{29.2} $\pm$ 1.6\\
Caltech101 &  28.8 & 25.6 & 44.3 $\pm$ 0.8 & 45.2 $\pm$ 0.6 & 39.6 $\pm$ 1.0 & 43.1 $\pm$ 0.8 & \textbf{47.3} $\pm$ 0.8 & \underline{47.1} $\pm$ 0.8\\
Youtube\footnotemark[1] & 24.1 & 24.8 & OM & OM & OM & \underline{39.9} $\pm$ 2.0  & OM & \textbf{42.3} $\pm$ 0.6\\
\midrule
\multicolumn{9}{c}{\textbf{NMI ($\%$)}} \\
\midrule
Caltech7  & 55.6 & 48.1  & 78.3 $\pm$ 3.8 & 79.7 $\pm$ 3.5 & 66.4 $\pm$ 2.7 & \underline{82.6} $\pm$ 3.5 & 81.2 $\pm$ 2.1 & \textbf{84.3} $\pm$ 4.2  \\
HW & 79.2 &  83.7 & 91.7 $\pm$ 1.1 & \textbf{94.2} $\pm$ 0.5 & 86.0 $\pm$ 2.6 & \underline{93.7} $\pm$ 1.2 & 93.3 $\pm$ 0.8 & 93.4 $\pm$ 1.0 \\
Caltech20   & 63.9 & 59.5 & 67.4 $\pm$ 2.6 & 68.6 $\pm$ 2.6 & 56.5 $\pm$ 3.7 & \underline{72.6} $\pm$ 1.7 & 69.9 $\pm$ 1.6 & \textbf{74.6} $\pm$ 1.5\\
Reuters &  35.9 & 35.6 & OM & OM & 36.3 $\pm$ 2.4 & 47.0 $\pm$ 2.0 & \underline{49.7} $\pm$ 0.9 & \textbf{50.9} $\pm$ 2.7 \\
NUS  & 13.5 &  12.5 & 12.2 $\pm$ 1.0 & OM & 6.7 $\pm$ 1.1 & \underline{14.4} $\pm$ 0.6 & 13.0 $\pm$ 0.2 & \textbf{17.7} $\pm$ 1.0\\
Caltech101  & 41.8 & \underline{48.5} & 41.5 $\pm$ 0.8 & 46.1 $\pm$ 0.8 & 38.1 $\pm$ 0.8 & 44.0 $\pm$ 0.7 & \textbf{50.9} $\pm$ 0.5 & \textbf{50.9} $\pm$ 0.6\\
Youtube & 24.5 & 22.2 & OM & OM & OM & \underline{26.8} $\pm$ 1.6 & OM & \textbf{27.6} $\pm$ 0.6 \\
\midrule
\multicolumn{9}{c}{\textbf{ARI ($\%$)}} \\
\midrule
Caltech7  & 41.3 & 43.0  & 87.6 $\pm$ 4.1 & 85.4 $\pm$ 6.0 & 65.4 $\pm$ 4.1 & \underline{92.6} $\pm$ 3.0 & 92.4 $\pm$ 2.1 &  \textbf{93.6} $\pm$ 4.1\\
HW  & 72.7 & 80.0 & 91.1 $\pm$ 2.2 & \textbf{94.1} $\pm$ 0.5 & 78.6 $\pm$ 4.0 & \underline{93.7} $\pm$ 1.5 & 93.5 $\pm$ 1.1 & 93.6 $\pm$ 1.1\\
Caltech20 & 63.6 & 33.3 & 80.6 $\pm$ 3.7 & 65.8 $\pm$ 6.5 & 39.0 $\pm$ 7.6 & \underline{84.6} $\pm$ 2.2 & 81.0 $\pm$ 2.9 & \textbf{88.8} $\pm$ 1.5\\
Reuters  & 33.2 & 22.1 & OM & OM & 33.1 $\pm$ 4.2 & 48.4 $\pm$ 2.1 & \underline{51.9} $\pm$ 1.7 & \textbf{55.1} $\pm$ 3.9\\
NUS  & 6.5 & 4.9 & 6.7 $\pm$ 1.8 & OM & 2.8 $\pm$ 1.6 & \underline{10.8} $\pm$ 0.7 & 10.7 $\pm$ 0.5 & \textbf{13.8} $\pm$ 2.2\\
Caltech101 & 17.4 & 19.1 & 41.5 $\pm$ 2.4 & 35.7 $\pm$ 3.5 & 20.6 $\pm$ 3.3 & 39.2 $\pm$ 1.9 & \underline{63.5} $\pm$ 1.8 & \textbf{63.8} $\pm$ 1.2\\
Youtube & 6.0 & 4.6 & OM & OM & OM & \underline{7.8} $\pm$ 2.1 & OM & \textbf{13.6} $\pm$ 0.6\\
\midrule
\multicolumn{9}{c}{\textbf{Running Time ($s$)}} \\
\midrule
Caltech7  & 23.2 & 251.7 & 4.6 & 10.9 & \textbf{0.9} & 8.4 & 12.9 & \underline{2.8} \\
HW &  25.5 & 8.6 & \underline{3.9} & 11.6 & \textbf{1.5} & 7.7 & 6.9 & 4.8\\
Caltech20  & 26.4 & 28.8 & 18.7 & 23.4 & \textbf{2.5} & 14.2 & 24.5 & \underline{6.3}\\
Reuters & 1342.1 & \underline{317.7} & OM & OM & 326.3 & 1206.3 & 5096.6 & \textbf{109.5}\\
NUS & 1113.2 & 1879.4 & 763.9 & OM & 990.4 & \underline{145.4} & 1431.4 & \textbf{40.7}\\
Caltech101  & 684.9 & 2558.1 & 137.1 & 412.8 & \underline{47.6} & 54.3 & 199.8 & \textbf{20.5}\\
Youtube & 3187.1 & 1152.7 & OM & OM & OM & \underline{1128.4} & OM & \textbf{209.8} \\
\bottomrule
\end{tabular*}}
\footnotetext[1]{Youtube is an abbreviation of the dataset YoutubeFace$\_$sel.}
\end{table}

In this section, we conduct extensive experiments on seven benchmark multi-view datasets to evaluate the effectiveness and efficiency of the proposed method.

\subsection{Experimental Settings}

\subsubsection{Datasets} We perform experiments on seven datasets widely employed for multi-view clustering: Handwritten (HW) \cite{misc_multiple_features_72}, 
Caltech101 \cite{li_andreeto_ranzato_perona_2022},  
Caltech7, Caltech20, Reuters \cite{amini2009learning}, NUS-WIDE-Object (NUS) \cite{chua2009nus}, and YoutubeFace$\_$sel \cite{wolf2011face,wang2021fast}. 
Specifically, HW comprises 2,000 handwritten digits from 0 to 9. Caltech101 is an object recognition dataset including 9,144 data points. Caltech7 and Caltech20 are two widely-used subsets of Caltech101. Reuters includes documents in five languages and their translations; we use the English-written subset and its translations, which consists of 18,758 data points in 6 categories. NUS is also a dataset for object recognition, comprising 30,000 data points in 31 categories. YoutubeFace$\_$sel is a video dataset containing 101,499 samples.
The key dataset statistics are presented in \cref{tab: multi-view datasets}.  

\subsubsection{Compared Methods}

We compare the proposed FSSMSC with
two scalable multi-view subspace clustering methods, including LMVSC \cite{kang2020large} and FPMVS \cite{wang2021fast}, and five state-of-the-art semi-supervised multi-view learning methods, including AMGL \cite{nie2016parameter}, MVAR \cite{tao2017scalable}, MLAN \cite{nie2017auto}, GPSNMF~\cite{liang2020semi}, and CONMF~\cite{liang2022co}.

\subsubsection{Evaluation Metrics} We evaluate the clustering performance with clustering accuracy (ACC) \cite{li2017robust}, Normalized Mutual Information (NMI) \cite{li2017robust}, Adjusted Rand Index (ARI) \cite{hubert1985comparing}, and running time in seconds. Since the proposed approach and the compared methods in the paper are implemented in MATLAB, we use the "tic" and "toc" functions to measure the elapsed time for each algorithm, with "tic" marking the start and "toc" marking the end of the execution of the algorithm. This elapsed time is reported as the algorithm's running time.

\subsubsection{Implementation Details} To ensure an equitable comparison, we acquired the released codes of the comparative methods and fine-tuned the parameters according to the corresponding papers to achieve optimal outcomes. 
In our approach, we consistently fix $\lambda$ and $\lambda_0$ as $100$, $m_Z$ as $1$, $C_q$ as $n$, $\epsilon$ and $\epsilon_q$ as $1\times 10^{-5}$, and $\eta_q$ and $\eta_z$ as $\frac{1}{\lambda}$. 
Moreover, we select $\lambda_M$ from $\{10^1,10^2,10^3,10^4,10^5\}$, $\lambda_Z$ from $\{0.01,0.02,0.05,0.07,0.1,0.2,0.5\}$, and $\beta$ from $\{10^{-5},10^{-4},10^{-3},10^{-2},10^{-1}\}$. The number of clusters is set as the number of data categories $k$. For LMVSC and FPMVS, the landmark number $m$ is explored from ${k,100,200,300,500}$, with the best clustering performance reported. For the proposed FSSMSC, $m=500$ is employed for most datasets, except for Caltech7 and Reuters, where $m=300$ is selected. For AMGL, MLAN, MVAR, GPSNMF, CONMF, and FSSMSC, labeled samples constitute $5\%$ of the total samples for HW, Caltech7, Caltech20, and Caltech101, and $1\%$ for the larger datasets, Reuters, NUS and YoutubeFace$\_$sel. Each experiment is repeated 20 times independently, and the average values and standard deviations are reported. All experiments are implemented in Matlab R2021a and run on a platform with a 2.10GHz Intel Xeon CPU and 128GB RAM.

\subsection{Experimental Results and Discussion}

\cref{tab: multi-view results} presents the clustering performance of different methods on the seven datasets. 
Based on the results depicted in Table \ref{tab: multi-view results}, the following observations can be made:
\begin{itemize}
\item Compared to the scalable multi-view subspace clustering methods, FPMVS and LMVSC, the proposed approach demonstrates a significant enhancement in clustering performance while utilizing a limited amount of supervisory information. For instance, FSSMSC exhibits improvements of $18.3\%$, $9.8\%$, and $17.5\%$ over datasets Reuters, NUS, and YoutubeFace$\_$sel, respectively, for the ACC metric. Furthermore, FSSMSC notably outperforms FPMVS and LMVSC regarding running time.

\item In comparison to the semi-supervised multi-view learning methods (MVAR, MLAN, AMGL, GPSNMF, and CONMF), the proposed FSSMSC demonstrates comparable or superior performance with remarkably enhanced computational efficiency, particularly on large datasets such as YoutubeFace$\_$sel, NUS, and Reuters.
Specifically, 
FSSMSC demonstrates improvements of $5.8\%$, $3.2\%$, and $3.0\%$ in ARI metric on datasets YoutubeFace$\_$sel, Reuters, and NUS, respectively, when contrasted with the five semi-supervised methods. 
Furthermore, FSSMSC exhibits computational efficiency advantages, running nearly 25 times faster than AMGL on NUS, 15 times faster than MVAR on NUS, 10 times faster than GPSNMF on Reuters, 20 times faster than MLAN on Caltech101, and 45 times faster than CONMF on Reuters.

 \item When conducting experiments on the largest dataset, YoutubeFace$\_$sel, the semi-supervised methods MVAR, MLAN, AMGL, and CONMF faced out-of-memory issues. The proposed FSSMSC demonstrates improvements of $2.4\%$ 
 and $5.8\%$ over GPSNMF in terms of the ACC and ARI metrics, respectively, achieving this enhancement with significantly reduced running time.
\end{itemize}

These experiments verify the scalability, effectiveness, and efficiency of the proposed method.

\begin{figure}[!ht]
\centering
\includegraphics[width=1.0\textwidth]{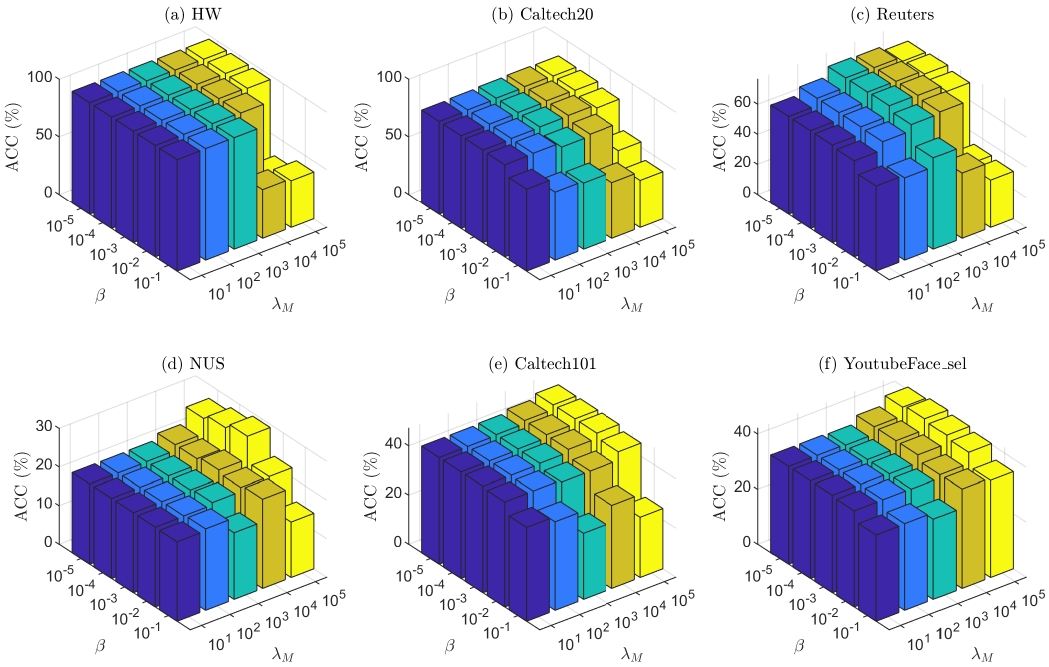}
\caption{Clustering accuracy with different value of $\lambda_M$ and $\beta$ on six datasets.}
\label{fig: lambdaMbeta sensitivity analysis on all datasets}
\end{figure}

\begin{figure}
\centering
\includegraphics[width=1.0\textwidth]{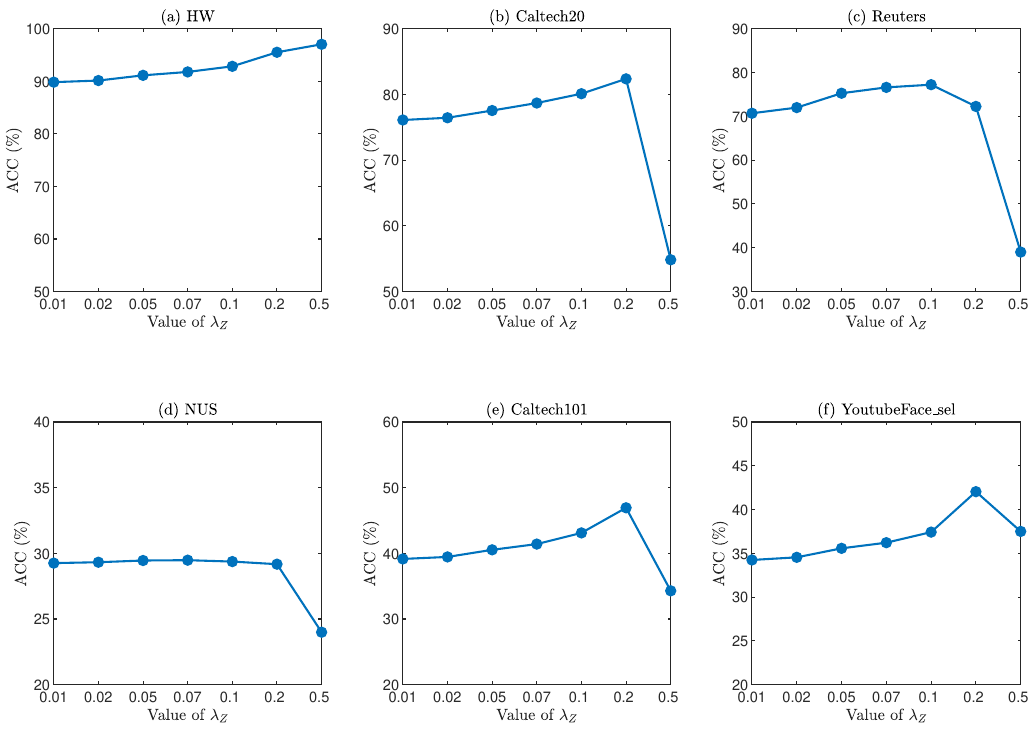}
\caption{Clustering accuracy with different value of $\lambda_Z$ on six datasets.}
\label{fig: lambdaZ sensitivity analysis on all datasets}
\end{figure}

\begin{figure}[!ht]
\centering
\includegraphics[width=1.0\textwidth]{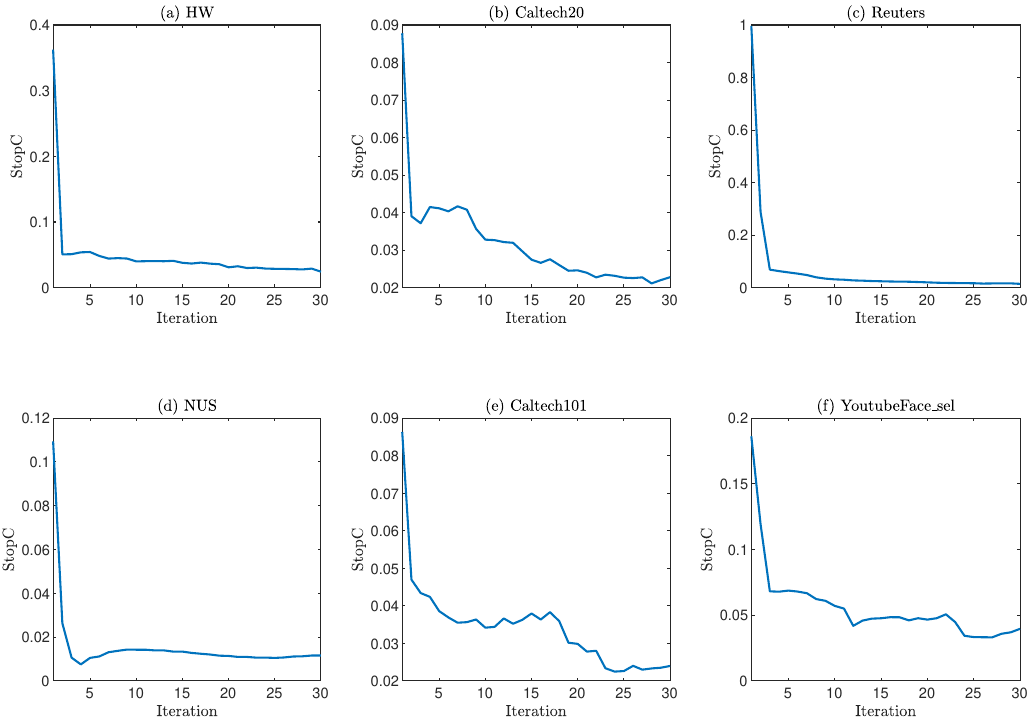}
\caption{The numerical convergence curves for six datasets, where the stopping criteria at iteration $j$ is defined as: $\text{StopC}=\max(\|\mathbf{B}^{j+1}-\mathbf{B}^{j}\|_\infty,\|\mathbf{Z}^{j+1}-\mathbf{Z}^{j}\|_\infty)$.}
\label{fig: convergence curves for all datasets}
\end{figure}

\subsection{Sensitivity and Convergence Analysis}

In \cref{fig: lambdaMbeta sensitivity analysis on all datasets} and \cref{fig: lambdaZ sensitivity analysis on all datasets}, we conduct a sensitivity analysis for the three key parameters in FSSMSC, namely $\lambda_Z$, $\lambda_M$, and $\beta$, on six datasets. The following observations can be made from \cref{fig: lambdaMbeta sensitivity analysis on all datasets} and \cref{fig: lambdaZ sensitivity analysis on all datasets}:
(1) The parameter $\lambda_Z$, governing the sparsity level of $\mathbf{Z}$, significantly influences the clustering performance. Notably, the clustering accuracy exhibits more stable changes when the value of $\lambda_Z$ is relatively small;
(2) Among the three relatively large datasets (Caltech101, NUS, and YoutubeFace$\_$sel), higher values of $\lambda_M$ such as $1 \times 10^5$ and smaller values of $\beta$ such as $1 \times 10^{-4}$ and $1 \times 10^{-3}$ are advantageous.

Additionally, numerical convergence curves for six datasets are presented in \cref{fig: convergence curves for all datasets}. It can be observed that the stopping criterion decreases rapidly in the initial five iterations and continues to decrease throughout the iterations. In our experiments, we set the maximum iteration number to 30.

\subsection{Comparison of Semi-Supervised Methods under Different Labeled Ratios}

We conduct experiments to assess the performance of our proposed method under different ratios of labeled samples. The results are presented in \cref{fig: ACC on all datasets with different labeled ratios}, \cref{fig: NMI on all datasets with different labeled ratios}, \cref{fig: ARI on all datasets with different labeled ratios}, and \cref{fig: caltech7 with different labeled ratios}. The results are not included when an out-of-memory issue occurs.
These results show that our proposed FSSMSC performs comparably to or better than semi-supervised competitors (MVAR, MLAN, AMGL, GPSNMF, and CONMF) across different labeled sample ratios.  Additionally, as presented in \cref{tab: multi-view results}, FSSMSC demonstrates notably reduced running times on the large datasets—NUS, Reuters, and YoutubeFace$\_$sel. These results further validate the effectiveness of our approach.

\begin{figure}[!t]
\centering
\includegraphics[width=1.0\textwidth]{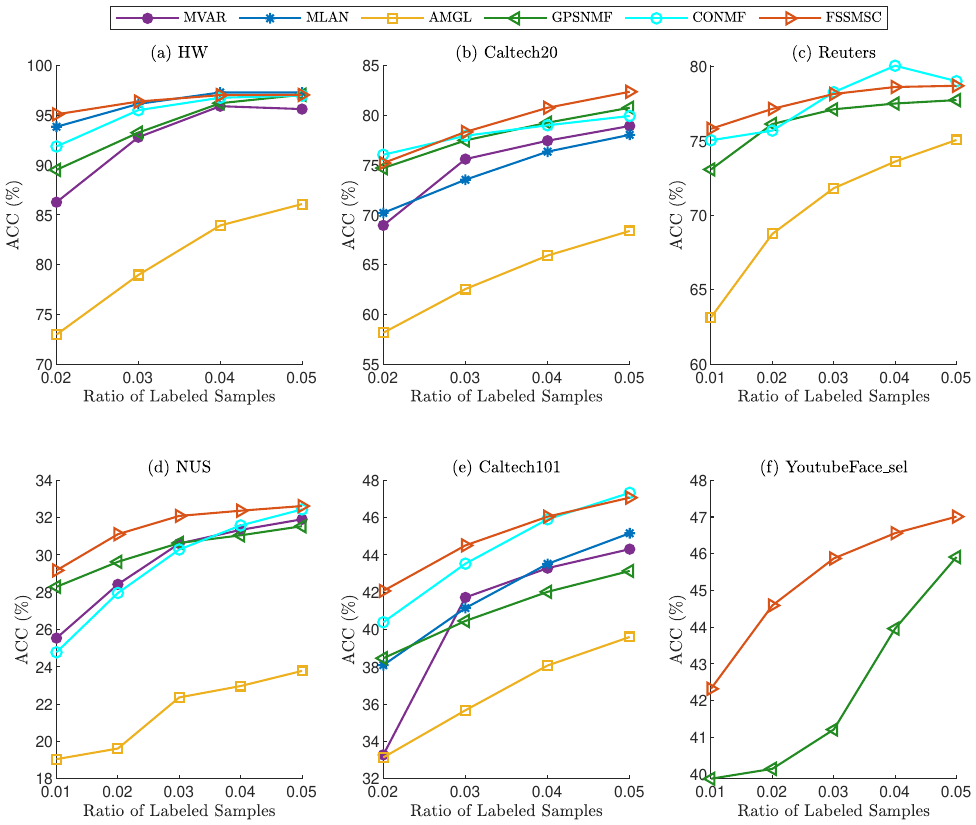}
\caption{The evaluation metric ACC with different ratio of labeled samples on six datasets.}
\label{fig: ACC on all datasets with different labeled ratios}
\end{figure}

\begin{figure}[!t]
\centering
\includegraphics[width=1.0\textwidth]{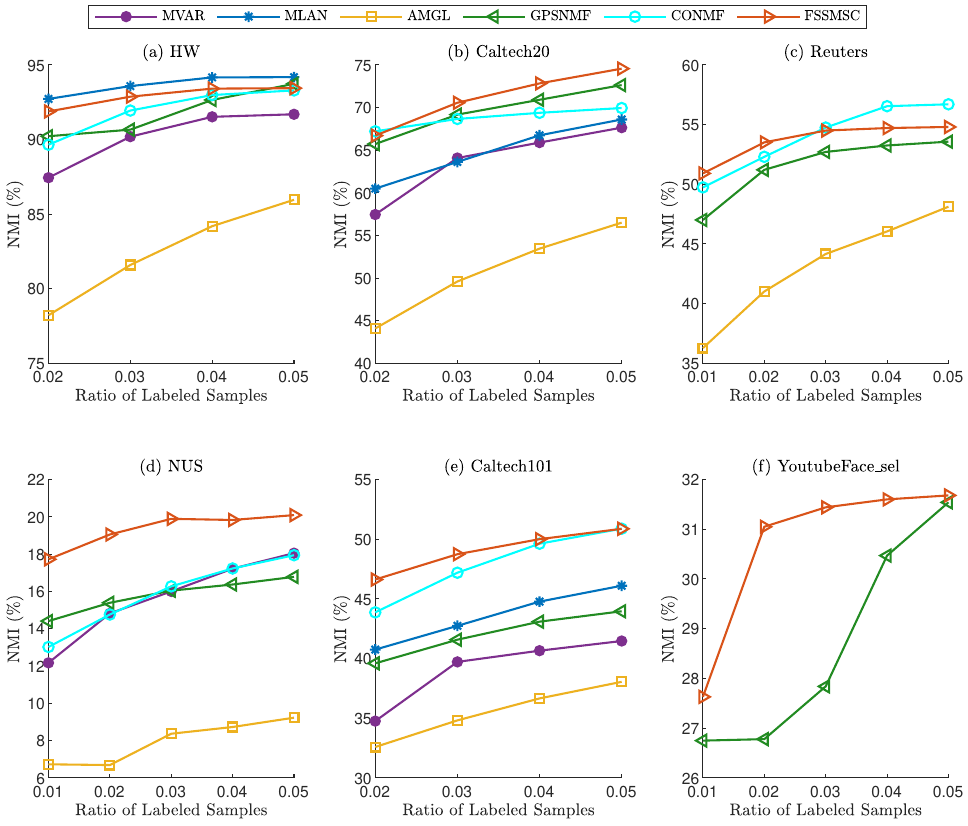}
\caption{The evaluation metric NMI with different ratio of labeled samples on six datasets.}
\label{fig: NMI on all datasets with different labeled ratios}
\end{figure}

\begin{figure}[!t]
\centering
\includegraphics[width=1.0\textwidth]{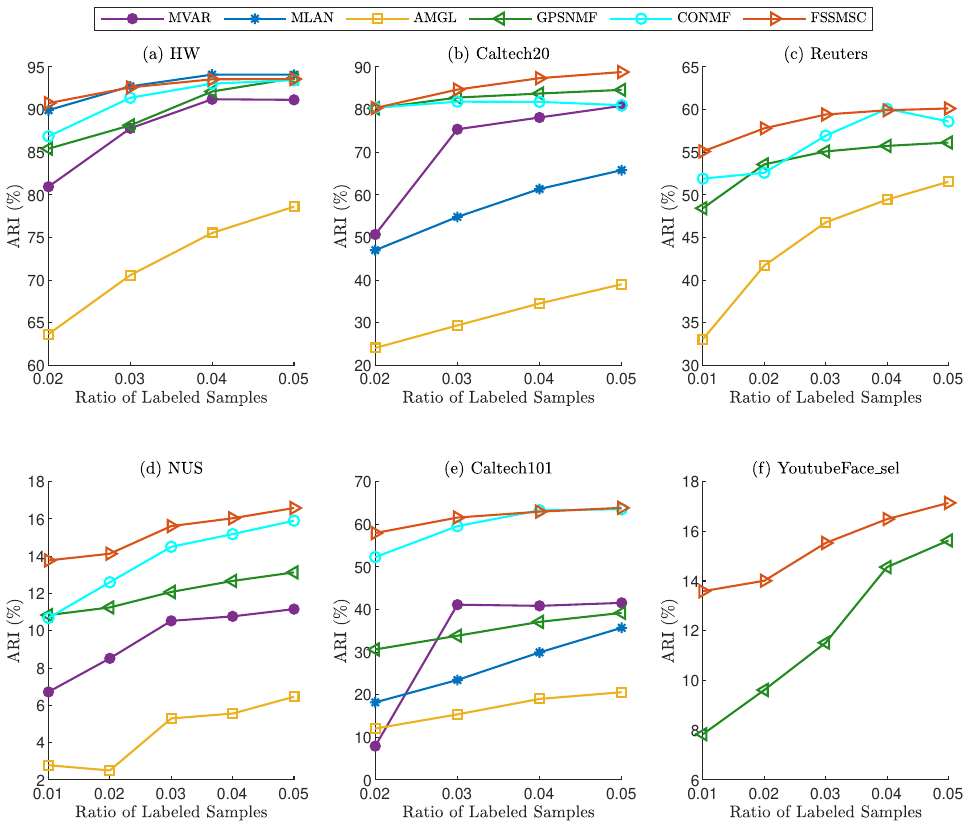}
\caption{The evaluation metric ARI with different ratio of labeled samples on six datasets.}
\label{fig: ARI on all datasets with different labeled ratios}
\end{figure}

\begin{figure}[!t]
\centering
\includegraphics[width=1.0\textwidth]{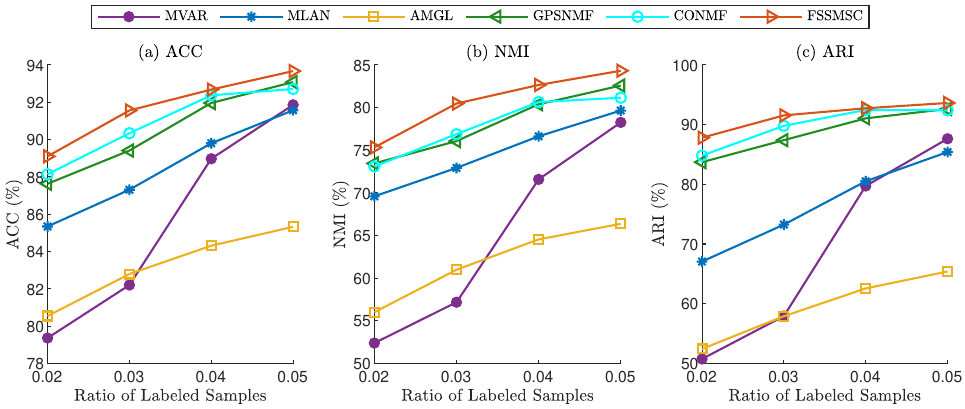}
\caption{The evaluation metrics with different ratio of labeled samples on the Caltech7 datasets.}
\label{fig: caltech7 with different labeled ratios}
\end{figure}

\subsection{Ablation Study}

We conduct an ablation study for each component of our optimization model \eqref{op: FSSMSC}. The results are presented in \cref{tab: ablation study complete}. These ablation experiments include:
\begin{itemize}
\item W/O $\|\mathbf{Z}\|_1$: Omitting the $\ell_1$ norm of $\mathbf{Z}$ in $\varphi(\mathbf{Z})$, resulting in $\varphi(\mathbf{Z}) = \frac{1}{2}\sum_{v=1}^V\|\mathbf{X}^{(v)}-\mathbf{U}^{(v)}\mathbf{Z}\|^2$. Since $\frac{1}{2}\sum_{v=1}^V\|\mathbf{X}^{(v)}-\mathbf{U}^{(v)}\mathbf{Z}\|^2$ is a necessary component to utilize the data features $\mathbf{X}^{(v)},v=1,\ldots,V$, we preserve this part in $\varphi(\mathbf{Z})$.
\item W/O $\phi_\mathcal{M}$: Excluding the component $\phi_\mathcal{M}(\mathbf{A},\mathbf{Z})$. 
\item W/O $\phi_\mathcal{C}$: Omitting the component $\phi_\mathcal{C}(\mathbf{A},\mathbf{Z})$. 
\item W/O $\phi_s$: Excluding the component $\phi_s(\mathbf{A},\mathbf{Z})$.
\end{itemize}
The results in \cref{tab: ablation study complete} demonstrate that the full FSSMSC model consistently achieves superior clustering performance in most cases compared to versions with individual components removed. This confirms the importance and effectiveness of each component in our proposed model.

\begin{table}[!ht]
\centering
\footnotesize
\caption{The clustering metrics ACC and NMI on different datasets. The best results are highlighted in boldface.
}
\label{tab: ablation study complete}
\begin{tabular}{@{}lcccccc@{}} 
\toprule
Dataset &  Reuters  & NUS & YoutubeFace$\_$sel & Caltech101 & Caltech20 & HW \\
\midrule
\multicolumn{7}{c}{\textbf{ACC($\%$)}} \\
\midrule
W/O $\|\mathbf{Z}\|_1$ & 68.1 $\pm$ 2.3 & \textbf{29.2} $\pm$ 1.2 & 34.5 $\pm$ 0.6 &   
        38.5 $\pm$ 0.6 & 75.8 $\pm$ 1.5 & 89.5 $\pm$ 1.9 \\
W/O $\phi_\mathcal{M}$ & 64.4 $\pm$ 2.4 & 20.8 $\pm$ 0.9 & 35.9 $\pm$ 0.8 &            
         43.4 $\pm$ 0.6 & 80.5 $\pm$ 1.5 & 96.9 $\pm$ 0.6  \\
W/O $\phi_\mathcal{C}$ & 69.5 $\pm$ 1.9 & 17.1 $\pm$ 1.3 & 24.7 $\pm$ 0.9 &            
          44.8 $\pm$ 0.7 & 82.2 $\pm$ 1.7 & \textbf{97.1} $\pm$ 0.5  \\
W/O $\phi_s$ & 58.0 $\pm$ 12.7 & 21.9 $\pm$ 4.5 & 30.2 $\pm$ 1.5 &                      
          42.9 $\pm$ 1.2 & 80.5 $\pm$ 4.2 & 86.8 $\pm$ 6.3 \\
FSSMSC &  \textbf{75.8} $\pm$ 2.4 & \textbf{29.2} $\pm$ 1.6 & \textbf{42.3} $\pm$ 0.6 & 
        \textbf{47.1} $\pm$ 0.8 & \textbf{82.4} $\pm$ 1.8 & \textbf{97.1} $\pm$ 0.5  \\
\midrule 
\multicolumn{7}{c}{\textbf{NMI($\%$)}} \\
\midrule
W/O $\|\mathbf{Z}\|_1$ & 41.3 $\pm$ 2.4 & 17.2 $\pm$ 1.0 & 21.2 $\pm$ 0.7 &            
           42.8 $\pm$ 0.6 & 64.7 $\pm$ 1.4 & 82.5 $\pm$ 1.5  \\
W/O $\phi_\mathcal{M}$ & 37.0 $\pm$ 2.6 & 11.0 $\pm$ 0.5 & 21.1 $\pm$ 0.6 &            
           47.3 $\pm$ 0.5 & 72.3 $\pm$ 1.4 & 93.2 $\pm$ 1.2  \\
W/O $\phi_\mathcal{C}$ & 44.2 $\pm$ 2.0 & 7.6 $\pm$ 0.9 & 9.7 $\pm$ 0.5 &              
           47.8 $\pm$ 0.7 & \textbf{75.1} $\pm$ 1.3 & \textbf{93.6} $\pm$ 0.9 \\
W/O $\phi_s$ & 28.1 $\pm$ 15.0 & 13.6 $\pm$ 4.5 & 16.9 $\pm$ 1.3 &                     
           46.8 $\pm$ 1.3 & 70.6 $\pm$ 6.4 & 77.8 $\pm$ 7.8  \\
FSSMSC &  \textbf{50.9} $\pm$ 2.7 & \textbf{17.7} $\pm$ 1.0 & \textbf{27.6} $\pm$ 0.6 & 
          \textbf{50.9} $\pm$ 0.6 & 74.6 $\pm$ 1.5 & 93.4 $\pm$ 1.0 \\
\bottomrule
\end{tabular}
\end{table}

Moreover, we perform an ablation study on the parameter $\beta$. FSSMSC$^*$ indicates the case where $\beta$ is fixed at 0 in \cref{alg: FSSMSC}, while FSSMSC denotes the usage of a tuned non-zero $\beta$. 
The principal distinction between FSSMSC$^*$ and FSSMSC lies in the $\mathbf{Z}$ update. In FSSMSC$^*$, $\mathbf{Z}$ is updated without utilizing the landmarks' representations $\mathbf{A}$. Conversely, in FSSMSC, $\mathbf{Z}$ is updated using both data features and landmarks' representations. 
Therefore, anchor graph learning in FSSMSC$^*$ is decoupled from the label propagation process.
The results of applying FSSMSC and FSSMSC$^*$ to different datasets are presented in \cref{tab: ablation study}. The outcomes depicted in \cref{tab: ablation study} illustrate that FSSMSC outperforms FSSMSC$^*$ in terms of ACC and NMI metrics. This observation verifies the efficacy of the simultaneous learning of the anchor graph and label propagation.

\begin{table}[!t]
\centering
\caption{The clustering metrics ACC and NMI on different datasets. FSSMSC$^*$ denotes the scenario where $\beta$ is set to 0 in \cref{alg: FSSMSC}.}
\label{tab: ablation study}
\begin{tabular}{@{}lccccc@{}} 
\toprule
Dataset &  Reuters  & NUS & YoutubeFace$\_$sel & Caltech101 & HW\\
\midrule
\multicolumn{6}{c}{\textbf{ACC($\%$)}} \\
\midrule
FSSMSC$^*$ &  75.2 $\pm$ 2.5 & 26.1 $\pm$ 1.1 & 41.0 $\pm$ 0.6 & 46.9 $\pm$ 0.7 & 97.0 $\pm$ 0.6\\
FSSMSC &  75.8 $\pm$ 2.4 & 29.2 $\pm$ 1.6 & 42.3 $\pm$ 0.6 & 47.1 $\pm$ 0.8 & 97.1 $\pm$ 0.5 \\
\midrule 
\multicolumn{6}{c}{\textbf{NMI($\%$)}} \\
\midrule
FSSMSC$^*$ &  49.9 $\pm$ 2.8 & 14.7 $\pm$ 0.6 & 26.1 $\pm$ 0.5 & 50.4 $\pm$ 0.6 & 93.3 $\pm$ 1.0\\
FSSMSC &  50.9 $\pm$ 2.7 & 17.7 $\pm$ 1.0 & 27.6 $\pm$ 0.6 & 50.9 $\pm$ 0.6 & 93.4 $\pm$ 1.0\\
\bottomrule
\end{tabular}
\end{table}

\subsection{Relationship Between Computational Complexity and Running Time}

In this section, we illustrate the relationship between the computational complexity and the growth speed of running time. 
For instance, the compared semi-supervised algorithms MLAN, AMGL, GPSNMF, and CONMF require constructing similarity matrices for all samples across all views. 
The computational complexity for constructing these similarity matrices is $O(dn^2)$, where $d$ is the total feature dimension of all views and $n$ is the number of samples.
This complexity scales quadratically with $n$, whereas our algorithm's overall complexity scales linearly with $n$. Therefore, as $n$ increases, the running time required to construct these similarity matrices grows significantly faster for these methods compared to our approach.

We provide empirical evidence for this by reporting the total running time of the semi-supervised methods, as well as the running time required to construct similarity matrices before the iterative update procedure, across four datasets, as detailed in \cref{tab: running time}. In \cref{tab: running time}, values in parentheses denote the running time for similarity matrix construction, while those outside parentheses represent the total running time. 
It can be observed that as $n$ increases, the running time required to construct the similarity matrices (values in parentheses) grows significantly faster compared to our FSSMSC.
For example, the data size of the YoutubeFace$\_$sel dataset is nearly 3.4 times that of the NUS dataset, yet the running time for constructing similarity matrices for GPSNMF on the YoutubeFace$\_$sel dataset (646.0 seconds) is nearly 15.3 times that of the NUS dataset (42.3 seconds). In contrast, the total running time of our algorithm on the YoutubeFace$\_$sel dataset (209.8 seconds) is approximately 5.2 times that of the NUS dataset (40.7 seconds). 

Notably, MVAR does not involve similarity matrix construction, and its computational complexity scales linearly with $n$ while quadratically with $d$. Variations in running times for similarity matrix construction on the same dataset among MLAN, AMGL, GPSNMF, and CONMF are due to differences in their respective construction methods.

\begin{table}[!ht]
\centering
\footnotesize
\caption{The running time (in seconds) of each semi-supervised method. Values in parentheses denote the running time for similarity matrix construction, while those outside parentheses represent the total running time. Note that the MVAR and our FSSMSC methods do not involve constructing $n \times n$ similarity matrices. The best total running times are highlighted in boldface. OM means Out-of-Memory.}
\label{tab: running time}
\setlength{\tabcolsep}{1.12mm}{\begin{tabular*}{\textwidth}{@{\extracolsep\fill}ll|cccccc}
\toprule
Dataset & data size $n$ & MVAR & MLAN & AMGL & GPSNMF & CONMF & FSSMSC \\
\midrule
Caltech20 & 2,386 & 18.7 & 23.4(1.4) & \textbf{2.5}(1.7) & 14.2(0.6) & 24.5(2.3) & 6.3 \\
Caltech101 & 9,144 & 137.1 & 412.8(20.8) & 47.6(27.1) & 54.3(4.1) & 199.8(8.6) & \textbf{20.5} \\
NUS & 30,000 & 763.9 & OM & 990.4(440.2) & 145.4(42.3) & 1431.4(105.5) & \textbf{40.7} \\
Youtube & 101,499 & OM & OM & OM & 1128.4(646.0) & OM & \textbf{209.8} \\
\bottomrule
\end{tabular*}}
\footnotetext[1]{Youtube is an abbreviation of the dataset YoutubeFace$\_$sel.}
\end{table}

\subsection{The Floating Point Operations Evaluation}

The proposed approach and the compared methods in the paper are implemented in MATLAB. However, MATLAB does not have an official function or program to report exact Floating Point Operations (FLOPs). We use a publicly available MATLAB program~\cite{hang2024count} to estimate the FLOPs on the four large datasets, as shown in \cref{tab: FLOPs}. From \cref{tab: FLOPs}, we can observe that our proposed FSSMSC achieves significantly fewer FLOPs compared to other semi-supervised competitors.

\begin{table}[!ht]
\centering
\footnotesize
\caption{The estimated Floating Point Operations (FLOPs) of the semi-supervised methods. The best results are in boldface. OM means Out-of-Memory.}
\label{tab: FLOPs}
\setlength{\tabcolsep}{1.14mm}{\begin{tabular*}{\textwidth}{@{\extracolsep\fill}lcccccc}
\toprule
Dataset  & MVAR  &  MLAN & AMGL & GPSNMF & CONMF & FSSMSC \\
\midrule
Caltech101  & $4.03 \times 10^{13}$ & $1.06 \times 10^{14}$ & $7.73 \times 10^{12}$ & $1.32 \times 10^{13}$ & $1.11 \times 10^{13}$ &  $\mathbf{6.20 \times 10^{11}}$ \\
Reuters & OM & OM & $1.03 \times 10^{14}$ & $1.59 \times 10^{14}$ & $5.41 \times 10^{13}$ & $\mathbf{1.63 \times 10^{12}}$ \\
NUS  & $6.76 \times 10^{13}$ & OM & $2.13 \times 10^{14}$ & $1.02 \times 10^{14}$ & $2.34 \times 10^{13}$ & $\mathbf{1.86 \times 10^{12}}$ \\
YoutubeFace$\_$sel  & OM & OM & OM & $1.18 \times 10^{15}$ & OM & $\mathbf{6.53 \times 10^{12}}$ \\
\bottomrule
\end{tabular*}}
\end{table}

It should be noted that while GPSNMF typically has more FLOPs than CONMF, it takes less running time. This discrepancy arises from the FLOPs calculation rules in the used program~\cite{hang2024count}. The program calculates the FLOPs of multiplying an $n\times p$ matrix $\mathbf{A}$ by a $p \times m$ matrix $\mathbf{B}$ as $2npm$, regardless of whether the matrix is represented as sparse or not. However, in the released codes of GPSNMF, the similarity matrix $\mathbf{S}^p$ and its sub-matrix $\mathbf{S}_u^p$ are represented in sparse format (encoding the row and column indexes for the nonzero elements). Therefore, the actual FLOPs of matrix multiplication with these matrices may be overestimated. This may explain why the running time of GPSNMF is less than that of CONMF.

Additionally, the semi-supervised competitors MLAN, AMGL, GPSNMF, and CONMF require constructing similarity matrices for all samples across all views before the iterative update procedure. The FLOPs for this process are estimated as $dn^2$, where $d$ is the total feature dimension of all views and $n$ is the number of samples. The FLOPs for this process on the Caltech101, NUS, and YoutubeFace$\_$sel datasets are shown in \cref{tab: FLOPs of constructing graph}. 
On the YoutubeFace$\_$sel dataset, the FLOPs for constructing similarity matrices ($2.19\times 10^{13}$) is nearly 3.4 times of the FLOPs of our proposed FSSMSC ($6.53\times 10^{12}$). Moreover, the FLOPs for constructing these similarity matrices scale quadratically with $n$, whereas our approach scales linearly with $n$. This means that the FLOPs of constructing similarity matrices for these methods grows significantly faster compared to our approach.

\begin{table}[!ht]
\centering
\footnotesize
\caption{The FLOPs of constructing $n\times n$ similarity matrices for all views.}
\label{tab: FLOPs of constructing graph}
\begin{tabular}{lccc}
\toprule
Dataset &  Caltech101 & NUS &  YoutubeFace$\_$sel  \\
 \midrule 
 data size $n$ & 9,144 & 30,000 & 101,499 \\
 total dimension $d$ & 3,766 & 639 & 2,125 \\
$dn^2$ & $3.15 \times 10^{11}$ & $5.75 \times 10^{11}$ &  $2.19 \times 10^{13}$ \\
\bottomrule
\end{tabular}
\end{table}

\subsection{Experimental Validation of Linear Space Complexity}

We conduct experiments to confirm the linear property of our algorithm's space complexity. As illustrated in \cref{tab:memory usage}, we evaluate the space complexity by examining the memory usage of the variables in the MATLAB implementation of our algorithm. The Reuters dataset is not included due to its extremely high feature dimension compared to the other six datasets. The results show that the values of bytes/n for our approach are comparable across different datasets, indicating that the space complexity empirically scales linearly with the data size $n$. Additionally, the observed fluctuations in the values of bytes/n are due to variations in feature dimensions and the number of selected landmarks across datasets.

\begin{table}[!ht]
    \centering
    \footnotesize
    \caption{The memory usage (in bytes) of our approach on different datasets.}
    \label{tab:memory usage}
    \setlength{\tabcolsep}{1.16mm}{\begin{tabular}{@{}lcccccc@{}} 
        \toprule
        Dataset & HW & Caltech7 & Caltech20 & Caltech101 & NUS &  YoutubeFace$\_$sel   \\
        \midrule 
        data size $n$ & 2,000 & 1,474 &  2,386 & 9,144 & 30,000 &  101,499\\
        total dimension $d$ & 649 & 3,766 & 3,766 & 3,766 & 639 &  2,125\\
        bytes & $1.43 \times 10^8$ & $1.10 \times 10^8$ & $2.39 \times 10^8$ & $8.26 \times 10^8$ & $1.74 \times 10^9$  & $7.05 \times 10^9$\\
        bytes/n & $7.15 \times 10^4$ & $7.46\times 10^4$ & $1.00 \times 10^5$ & $9.03 \times 10^4$ & $5.80 \times 10^4$  & $6.95 \times 10^4$\\
        \bottomrule
    \end{tabular}}
\end{table}

For comparison, since each element of the matrix consumes 8 bytes, the $n \times n$ similarity matrices for all five views of the YoutubeFace$\_$sel dataset takes bytes of $5 \times 101,499 \times 101,499 \times 8>4.12 \times 10^{11}$, which is significantly large than the $7.05 \times 10^9$ bytes required by our approach. Furthermore, the memory usage for a full $n\times n$ similarity matrix grows quadratically with the data size $n$, whereas the memory usage for our approach grows linearly with $n$.

\section{Conclusion}

This paper introduces the Fast Scalable Semi-supervised Multi-view Subspace Clustering (FSSMSC) method, a novel approach addressing the prevalent high computational complexity in existing approaches. Importantly, FSSMSC offers linear computational and space complexity, rendering it suitable for large-scale multi-view datasets. Unlike conventional methods that manage anchor graph construction and label propagation independently, we propose a unified optimization model for simultaneously learning both components. Extensive experiments on seven multi-view datasets demonstrate the efficiency and effectiveness of the proposed method.

\backmatter





\bmhead{Acknowledgements}

We appreciate the support by the National
Natural Science Foundation of China (No.12271291, No.12071244, and No.92370125) and the National Key R$\&$D Program of China (No.2021YFA1001300).








\begin{appendices}

\section{Proofs of \cref{lemma: main lemma main} and \cref{theorem: main theorem main}}\label{secA1}


We begin by presenting essential notations and definitions, followed by the proofs for \cref{lemma: main lemma main} and \cref{theorem: main theorem main} in the paper, denoted as \cref{lemma: main lemma} and \cref{theorem: main theorem} in this appendix.

Let $\mathcal{S}_0=\{\mathbf{A}\in \mathbb{R}^{k\times m}| \mathbf{A}\mathbf{A}^\top=\mathbf{I}\}$, $\mathcal{S}_1 = \{\mathbf{Z}\in \mathbb{R}^{m\times n}| m_Z\geq \mathbf{Z}\geq \mathbf{0}\}$, and $\mathcal{S}_2=\{\mathbf{q}\in \mathbb{R}^m| C_q \geq \mathbf{q}\geq \epsilon_q\}$, along with their indicator functions denoted as $\ell_{\mathcal{S}_0}(\cdot)$, $\ell_{\mathcal{S}_1}(\cdot)$, and $\ell_{\mathcal{S}_2}(\cdot)$.

Let $h(\mathbf{B})=\frac{1}{2}\sum_{v=1}^V\|\mathbf{X}^{(v)}-\mathbf{U}^{(v)}\mathbf{B}\|^2$ and define the functions $g(\mathbf{A},\mathbf{Z},\mathbf{q})$ and $g_0(\mathbf{Z},\mathbf{q})$ as follows:
\begin{align}
\label{eq: definition of g and g0}
g(\mathbf{A},\mathbf{Z},\mathbf{q}) &= \beta \cdot \frac{\left( \text{Tr}(\mathbf{A}(\mathbf{I}-\mathbf{D}_\mathbf{q}^{-1/2}\mathbf{ZZ}^\top\mathbf{D}_\mathbf{q}^{-1/2})\mathbf{A}^\top) + \lambda_M \text{Tr}(\mathbf{A}\mathbf{Z}_\ell\mathbf{L}_\mathbf{M}\mathbf{Z}_\ell^\top\mathbf{A}^\top)\right)}{\text{Tr}(\mathbf{A}\mathbf{Z}_\ell\mathbf{L}_\mathbf{C}\mathbf{Z}_\ell^\top\mathbf{A}^\top)+\epsilon}\\
g_0(\mathbf{Z},\mathbf{q}) &= \frac{\lambda_0}{2}\|\mathbf{q}-\mathbf{ZZ}^\top \mathbf{1}_m\|^2 \nonumber
\end{align}
Additionally, let
\begin{align*}
f_0(\mathbf{A})&=\ell_{\mathcal{S}_0}({\mathbf{A}}),\\
f_1(\mathbf{Z})&=\lambda_Z \text{Tr}(\mathbf{Z}^\top\mathbf{1}_{m\times n}) + \ell_{\mathcal{S}_1}(\mathbf{Z}), \\
f_2(\mathbf{q})&=\ell_{\mathcal{S}_2}(\mathbf{q}).
\end{align*}
Then, we can rewrite the objective as 
\begin{align}
\label{eq: eq: definition of phi}
\phi(\mathbf{A},\mathbf{Z},\mathbf{q},\mathbf{B})=f(\mathbf{A},\mathbf{Z},\mathbf{q})+h(\mathbf{B})
\end{align}
where 
\begin{align}
\label{eq: definition of f}
f(\mathbf{A},\mathbf{Z},\mathbf{q})=g(\mathbf{A},\mathbf{Z},\mathbf{q}) + g_0(\mathbf{Z},\mathbf{q}) +f_0(\mathbf{A})+f_1(\mathbf{Z})+f_2(\mathbf{q}).
\end{align}

Then, the augmented Lagrangian is defined as:
\begin{align}
\label{eq: augmented Lagrangian}
\mathcal{L}_\lambda(\mathbf{A},\mathbf{Z},\mathbf{q},\mathbf{B},\Lambda):=\phi(\mathbf{A},\mathbf{Z},\mathbf{q},\mathbf{B}) + <\Lambda, \mathbf{B}-\mathbf{Z}> + \frac{\lambda}{2}\|\mathbf{B}-\mathbf{Z}\|^2,
\end{align}
where $\Lambda$ is the Lagrangian multiplier.

In the paper, we denote $\{\mathbf{A}^k,\mathbf{Z}^k,\mathbf{q}^k,\mathbf{B}^k,\Lambda^k\}$ as the variables generated by \cref{alg: FSSMSC} at iteration $k$.

The objective for the sub-problem of updating $\mathbf{Z}^{k+1}$ is defined as:
\begin{align*}
\mathcal{L}_{\mathbf{Z}}^{k+1}(\mathbf{Z}) = P^{k+1}(\mathbf{Z}) + \ell_{\mathcal{S}_1}(\mathbf{Z}),
\end{align*}
where
\begin{align}
\label{eq:P proof}
P^{k+1}(\mathbf{Z}):= \lambda_Z \text{Tr}(\mathbf{Z}^\top\mathbf{1}_{m\times n}) + g_0(\mathbf{Z},\mathbf{q}^k) + g(\mathbf{A}^{k+1},\mathbf{Z},\mathbf{q}^k) - <\Lambda^k,\mathbf{Z}> + \frac{\lambda}{2}\|\mathbf{Z}-\mathbf{B}^k\|^2
\end{align}
The objective for the sub-problem of updating $\mathbf{q}^{k+1}$ is defined as:
\begin{align*}
\mathcal{L}^{k+1}_{\mathbf{q}}(\mathbf{q}):=Q^{k+1}(\mathbf{q}) + \ell_{\mathcal{S}_2}(\mathbf{q}),
\end{align*}
where
\begin{align}
\label{eq:Q proof}
Q^{k+1}(\mathbf{q})=g_0(\mathbf{Z}^{k+1},\mathbf{q}) + g(\mathbf{A}^{k+1},\mathbf{Z}^{k+1},\mathbf{q})
\end{align}

\begin{definition}[\cite{attouch2013convergence}] 
\label{def: prox}
Let $f: \mathbb{R}^n \rightarrow \mathbb{R} \cup\{+\infty\}$ be a proper lower semi-continuous function which is lower bounded, and $\lambda$ a positive parameter. The proximal correspondence $\operatorname{Prox}_{\lambda f}: \mathbb{R}^n \rightrightarrows \mathbb{R}^n$, is defined through the formula
\begin{align*}
\operatorname{Prox}_{\lambda f}(x):=\operatorname{argmin}_{y \in \mathbb{R}^n} f(y)+\frac{1}{2 \lambda}\|y-x\|^2.
\end{align*}
\end{definition}

\begin{definition}
For any nonempty closed subset $C$ of $\mathbb{R}^n$, the projection on $C$ is the following point-to-set mapping:
\begin{align*}
P_C(\mathbf{x}):=\operatorname{argmin}_{\mathbf{z}\in C} ~\|\mathbf{x}-\mathbf{z}\|.
\end{align*}
\end{definition}

Recall that for any closed subset $C$ of $\mathbb{R}^n$,  its indicator function $\ell_C$ is defined by
\begin{align*}
\ell_C(\mathbf{x})=\begin{cases} 0   & if ~~\mathbf{x} \in C\\
+\infty & otherwise,
\end{cases}
\end{align*}
where $\mathbf{x}\in \mathbb{R}^n$. 

Then, for any nonempty closed subset $C$ of $\mathbb{R}^n$ and any positive $\lambda$, we have the equality
\begin{align*}
\operatorname{Prox}_{\lambda \ell_C}(\mathbf{x}) = P_C(\mathbf{x}).
\end{align*}

Consequently, our alternating iteration algorithm (\cref{alg: FSSMSC}) at the $k$-th iteration is as follows:
\begin{align}
\label{alg: alternating iteration}
\mathbf{A}^{k+1} &\leftarrow \text{argmin}_{\mathbf{A}}~ \mathcal{L}_\lambda(\mathbf{A},\mathbf{Z}^k,\mathbf{q}^k,\mathbf{B}^k,\Lambda^k) \nonumber \\
\mathbf{Z}^{k+1} &\leftarrow \operatorname{Prox}_{\eta_z \ell_{\mathcal{S}_1}}(\mathbf{Z}^k-\eta_z\nabla_\mathbf{Z} P^{k+1}(\mathbf{Z}^k)) \nonumber \\
\mathbf{q}^{k+1} &\leftarrow \operatorname{Prox}_{\eta_q \ell_{\mathcal{S}_2}}(\mathbf{q}^k-\eta_q\nabla_\mathbf{q}Q^{k+1}(\mathbf{q}^k)) \nonumber\\
\mathbf{B}^{k+1} &\leftarrow \text{argmin}_{\mathbf{B}}~ \mathcal{L}_\lambda(\mathbf{A}^{k+1},\mathbf{Z}^{k+1},\mathbf{q}^{k+1},\mathbf{B},\Lambda^k) \\
\Lambda^{k+1} &\leftarrow \Lambda^{k} + \lambda (\mathbf{B}^{k+1}-\mathbf{Z}^{k+1}) \nonumber \\
 k &\leftarrow k+1. \nonumber
\end{align}

\begin{definition}[Lipschitz Continuity] 
A function $f$ is Lipschitz continuous on the set $C$, if there exists a constant $L>0$ such that 
\begin{align*}
\|f(\mathbf{x}_1)-f(\mathbf{x}_2)\| \leq L \|\mathbf{x}_1-\mathbf{x}_2\|, ~\forall \mathbf{x}_1,\mathbf{x}_2\in C.
\end{align*}
$L$ is called the Lipschitz constant.
\end{definition}

\begin{lemma} 
\label{lemma: main lemma}
Let $h(\mathbf{B})=\frac{1}{2}\sum_{v=1}^V\|\mathbf{X}^{(v)}-\mathbf{U}^{(v)}\mathbf{B}\|^2$ and $P^{k+1},Q^{k+1}$ are given in \eqref{eq:P proof} and \eqref{eq:Q proof}. Then  $\nabla_\mathbf{B} h(\mathbf{B})$ is Lipschitz continuous with a Lipschitz constant $L_h>0$. $\nabla_\mathbf{Z} P^{k+1}(\mathbf{Z})$ is Lipschitz continuous in $\mathcal{S}_1$ with a Lipschitz constant $\hat{L}_Z>0$. $\nabla_\mathbf{q} Q^{k+1}(\mathbf{q})$ is Lipschitz continuous in $\mathcal{S}_2$ with a Lipschitz constant $\hat{L}_q>0$.
\end{lemma}

\begin{proof}
Applying \cref{lemma: Lipschitz continuous gradient of h}, we directly obtain that $\nabla_\mathbf{B} h(\mathbf{B})$ is Lipschitz continuous with a Lipschitz constant $L_h>0$.

As for the Lipschitz continuity of $\nabla_{\mathbf{q}}Q^{k+1}(\mathbf{q})$, according to the updating rules for $\mathbf{A}^{k+1}$ and $\mathbf{Z}^{k+1}$ in \eqref{alg: alternating iteration}, we have $\mathbf{A}^{k+1}\in \mathcal{S}_0$ and $\mathbf{Z}^{k+1}\in \mathcal{S}_1$. Then, according to \cref{lemma: Lipschitz continuous gradient of g wrt q}, we have $\nabla_\mathbf{q} g(\mathbf{A}^{k+1},\mathbf{Z}^{k+1},\mathbf{q})$ is Lipschitz continuous on $\mathcal{S}_2$ with a Lipschitz constant $L_q$. In addition, using the fact that $\nabla_{\mathbf{q}}g_0(\mathbf{Z}^{k+1},\mathbf{q})=\lambda_0(\mathbf{q}-\mathbf{Z}^{k+1}\mathbf{Z}^{k+1^\top}\mathbf{1}_m)$, we can directly obtain that 
$\nabla_{\mathbf{q}}g_0(\mathbf{Z}^{k+1},\mathbf{q})$ is Lipschitz continuous with a Lipschitz constant $\lambda_0$. Therefore, $\nabla_{\mathbf{q}}Q^{k+1}(\mathbf{q})$, where $Q^{k+1}(\mathbf{q})=g_0(\mathbf{Z}^{k+1},\mathbf{q})+g(\mathbf{A}^{k+1},\mathbf{Z}^{k+1},\mathbf{q})$ is defined in \eqref{eq:Q proof}, is Lipschitz continuous on $\mathcal{S}_2$ with a Lipschitz constant $\hat{L}_q$, where $\hat{L}_q:=L_q+\lambda_0$.

As for the Lipschitz continuity of $\nabla_{\mathbf{Z}}P^{k+1}(\mathbf{Z})$, according to the updating rule for $\mathbf{q}^{k}$ in \eqref{alg: alternating iteration}, we have $\mathbf{q}^k\in \mathcal{S}_2$ and $C_q\geq\mathbf{q}^k\geq \epsilon_q$.  According to \cref{lemma: Lipschitz continuous gradient of g wrt Z}, we have $\nabla_\mathbf{Z} g(\mathbf{A}^{k+1},\mathbf{Z},\mathbf{q}^k)$ is Lipschitz continuous with respect to $\mathbf{Z}$ on $\mathcal{S}_1$ with a Lipschitz constant $L_Z>0$. According to \cref{lemma: Lipschitz continuous gradient of g0 wrt Z}, we have $\nabla_\mathbf{Z} g_0(\mathbf{Z},\mathbf{q}^k)$ is Lipschitz continuous with respect to $\mathbf{Z}$ on $\mathcal{S}_1$ with a Lipschitz constant $C_{g_0}>0$. Consequently, applying the definition of $P^{k+1}(\mathbf{Z})$ in \eqref{eq:P proof}, we directly obtain $\nabla_\mathbf{Z} P^{k+1}(\mathbf{Z})$ is Lipschitz continuous with respect to $\mathbf{Z}$ on $\mathcal{S}_1$ with a Lipschitz constant $\hat{L}_Z>0$, where $\hat{L}_Z:=L_Z+C_{g_0}+\lambda$.

Therefore, we complete the proof.
\end{proof}

\begin{theorem} 
\label{theorem: main theorem}
If $\lambda> L_h+2+2L_h^2$, $\eta_q<\frac{1}{\hat{L}_q}$, and $\eta_z<\frac{1}{\hat{L}_Z}$, the following properties hold:

(1) The augmented Lagrangian $\mathcal{L}_\lambda(\mathbf{A}^{k},\mathbf{Z}^{k},\mathbf{q}^k,\mathbf{B}^k,\Lambda^k)$ in \eqref{eq: augmented Lagrangian} is lower bounded;

(2) There exists a constant $C>0$ such that
\begin{align}
&\mathcal{L}_\lambda(\mathbf{A}^{k},\mathbf{Z}^{k},\mathbf{q}^k,\mathbf{B}^k,\Lambda^k) \nonumber \\
-& \mathcal{L}_\lambda(\mathbf{A}^{k+1},\mathbf{Z}^{k+1},\mathbf{q}^{k+1},\mathbf{B}^{k+1},\Lambda^{k+1}) \nonumber \\
\geq &C \left(\|\mathbf{B}^{k+1}-\mathbf{B}^k\|^2+ \|\mathbf{q}^{k+1}-\mathbf{q}^{k}\|^2\right), \forall k\in \mathbb{N};
\end{align}

(3) When $k\rightarrow \infty$, we have
\begin{align}
&\|\mathbf{B}^{k+1}-\mathbf{B}^k\|\rightarrow 0,~~\|\mathbf{Z}^{k+1}-\mathbf{Z}^k\|\rightarrow 0, \nonumber \\ &\|\Lambda^{k+1}-\Lambda^k\|\rightarrow 0,~~\|\mathbf{q}^{k+1}-\mathbf{q}^{k}\| \rightarrow 0. 
\end{align}
\end{theorem}

\begin{proof}

Firstly, we provide the proof of part (1). Note that $\mathbf{A}^k\in\mathcal{S}_0,\forall k\in \mathbb{N}$ and $\mathcal{S}_0$ is a bounded set, implying the existence of constant $C_A>0$ such that $\|\mathbf{A}^k\| \leq C_A,\forall k\in \mathbb{N}$. Considering the definition of $\mathcal{S}_2$ and $\mathbf{q}^k\in \mathcal{S}_2, \forall k\in \mathbb{N}$, we have $\mathbf{q}^k \geq \epsilon_q, \forall k\in \mathbb{N}$.

Let $g_2(\mathbf{A},\mathbf{Z})=\text{Tr}(\mathbf{A}\mathbf{Z}_\ell\mathbf{L}_\mathbf{C}\mathbf{Z}_\ell^\top\mathbf{A}^\top)+\epsilon$.
Moreover, since $\mathbf{L}_\mathbf{M}$ and $\mathbf{L}_\mathbf{C}$ are two given positive semi-definite matrices, we have $g_2(\mathbf{A},\mathbf{Z})\geq \epsilon>0$. Let $\mathbf{A}=\mathbf{A}^{k}$, $\mathbf{q}=\mathbf{q}^k$, and $\mathbf{Z}=\mathbf{Z}^k$. Then, we derive that 
\begin{align*}
g(\mathbf{A},\mathbf{Z},\mathbf{q}) &= \beta \cdot \frac{\left( \text{Tr}(\mathbf{A}(\mathbf{I}-\mathbf{D}_\mathbf{q}^{-1/2}\mathbf{ZZ}^\top\mathbf{D}_\mathbf{q}^{-1/2})\mathbf{A}^\top) + \lambda_M \text{Tr}(\mathbf{A}\mathbf{Z}_\ell\mathbf{L}_\mathbf{M}\mathbf{Z}_\ell^\top\mathbf{A}^\top)\right)}{g_2(\mathbf{A},\mathbf{Z})} \\
& \geq \beta \cdot \frac{\text{Tr}(\mathbf{A}\mathbf{D}_\mathbf{q}^{-1/2}(\mathbf{D}_\mathbf{q}-\mathbf{ZZ}^\top)\mathbf{D}_\mathbf{q}^{-1/2}\mathbf{A}^\top)}{g_2(\mathbf{A},\mathbf{Z})} \\
& = \beta \cdot \frac{\text{Tr}(\mathbf{A}\mathbf{D}_\mathbf{q}^{-1/2}(\mathbf{D}_\mathbf{q}-\mathbf{D}_\mathbf{Z}+\mathbf{D}_\mathbf{Z}-\mathbf{ZZ}^\top)\mathbf{D}_\mathbf{q}^{-1/2}\mathbf{A}^\top)}{g_2(\mathbf{A},\mathbf{Z})} \\
& \geq \beta \cdot \frac{\text{Tr}(\mathbf{A}\mathbf{D}_\mathbf{q}^{-1/2}(\mathbf{D}_\mathbf{q}-\mathbf{D}_\mathbf{Z})\mathbf{D}_\mathbf{q}^{-1/2}\mathbf{A}^\top)}{g_2(\mathbf{A},\mathbf{Z})} \\
& \geq -\beta \cdot \frac{\|\mathbf{A}\|^2\|\mathbf{D}_\mathbf{q}-\mathbf{D}_\mathbf{Z}\|\|\mathbf{D}_\mathbf{q}^{-1/2}\|^2}{g_2(\mathbf{A},\mathbf{Z})} \\
& \geq -\frac{m \beta C_A^2}{\epsilon\epsilon_q}\|\mathbf{q}-\mathbf{ZZ}^\top\mathbf{1}_m\|,
\end{align*} 
where $\mathbf{D}_\mathbf{Z}=diag(\mathbf{ZZ}^\top\mathbf{1}_m)$. 
Since $\mathbf{L}_\mathbf{M}$ is a positive semi-definite matrix, we have $\text{Tr}(\mathbf{A}\mathbf{Z}_\ell\mathbf{L}_\mathbf{M}\mathbf{Z}_\ell^\top\mathbf{A}^\top)=\text{Tr}(\mathbf{A}\mathbf{Z}_\ell\mathbf{L}_\mathbf{M}(\mathbf{A}\mathbf{Z}_\ell)^\top) \geq 0$. Additionally, since $\mathbf{D}_\mathbf{q}$ is a diagonal matrix, we can obtain $\mathbf{D}_\mathbf{q}^{-1/2} \mathbf{D}_\mathbf{q} \mathbf{D}_\mathbf{q}^{-1/2} = \mathbf{I}$. Then, we can derive 
\begin{align*}
\text{Tr}(\mathbf{A}(\mathbf{I}-\mathbf{D}_\mathbf{q}^{-1/2}\mathbf{ZZ}^\top\mathbf{D}_\mathbf{q}^{-1/2})\mathbf{A}^\top)=\text{Tr}(\mathbf{A}\mathbf{D}_\mathbf{q}^{-1/2}(\mathbf{D}_\mathbf{q}-\mathbf{ZZ}^\top)\mathbf{D}_\mathbf{q}^{-1/2}\mathbf{A}^\top).
\end{align*}
Consequently, we obtain the first inequality.  
The second inequality relies on the positive semi-definite property of the Laplacian matrix $\mathbf{D}_\mathbf{Z}-\mathbf{ZZ}^\top$. Consequently, we obtain 
\begin{align}
\label{ineq:lowerbound}
g(\mathbf{A},\mathbf{Z},\mathbf{q})+g_0(\mathbf{Z},\mathbf{q}) &\geq \frac{\lambda_0}{2}\|\mathbf{q}-\mathbf{ZZ}^\top\mathbf{1}_m\|^2 -\frac{m \beta C_A^2}{\epsilon\epsilon_q}\|\mathbf{q}-\mathbf{ZZ}^\top\mathbf{1}_m\| \nonumber \\
&=\frac{\lambda_0}{2}\left(\|\mathbf{q}-\mathbf{ZZ}^\top\mathbf{1}_m\|-\frac{m \beta C_A^2}{\lambda_0\epsilon\epsilon_q}\right)^2 - \frac{\lambda_0}{2}\left(\frac{m \beta C_A^2}{\lambda_0\epsilon\epsilon_q}\right)^2 \nonumber\\
&\geq  - \frac{\lambda_0}{2}\left(\frac{m \beta C_A^2}{\lambda_0\epsilon\epsilon_q}\right)^2.
\end{align}
Consequently, we obtain the lower bound of $f(\mathbf{A}^k,\mathbf{Z}^k,\mathbf{q}^k)$ defined in \eqref{eq: definition of f} as follows
\begin{align*}
f(\mathbf{A}^k,\mathbf{Z}^k,\mathbf{q}^k) \geq g(\mathbf{A}^k,\mathbf{Z}^k,\mathbf{q}^k)+g_0(\mathbf{Z}^k,\mathbf{q}^k)\geq  - \frac{\lambda_0}{2}\left(\frac{m \beta C_A^2}{\lambda_0\epsilon\epsilon_q}\right)^2.
\end{align*}
According to the definition of $h$, we have $h \geq 0$. Hence, we obtain
\begin{align*}
&\mathcal{L}_\lambda(\mathbf{A}^{k},\mathbf{Z}^{k},\mathbf{q}^k,\mathbf{B}^k,\Lambda^k) \\
=& f(\mathbf{A}^k,\mathbf{Z}^k,\mathbf{q}^k)+h(\mathbf{B}^k) +<\Lambda^k,\mathbf{B}^k-\mathbf{Z}^k>+\frac{\lambda}{2}\|\mathbf{B}^k-\mathbf{Z}^k\|^2 \\
= & f(\mathbf{A}^k,\mathbf{Z}^k,\mathbf{q}^k)+h(\mathbf{Z}^k)+h(\mathbf{B}^k)-h(\mathbf{Z}^k) - <\nabla h(\mathbf{B}^k),\mathbf{B}^k-\mathbf{Z}^k>+\frac{\lambda}{2}\|\mathbf{B}^k-\mathbf{Z}^k\|^2 \\
\geq & f(\mathbf{A}^k,\mathbf{Z}^k,\mathbf{q}^k)+h(\mathbf{Z}^k)+\frac{\lambda-L_h}{2}\|\mathbf{B}^k-\mathbf{Z}^k\|^2 \\
> & -\infty,
\end{align*}
where the first equality relies on \cref{lemma: control dual by primal}(a): $\Lambda^k=-\nabla h(\mathbf{B}^k)$ and the first inequality utilizes \cref{lemma: main lemma} and \cref{lemma: Lipschitz continuous gradient inequality}:
\begin{align*}
|h(\mathbf{Z}^k)-h(\mathbf{B}^k) - <\nabla h(\mathbf{B}^k),\mathbf{Z}^k-\mathbf{B}^k>| \leq \frac{L_h}{2}\|\mathbf{Z}^k-\mathbf{B}^k\|^2.
\end{align*}
Therefore, we complete the proof of part (1).

Secondly, we provide the proof of part (2). Given $\lambda> L_h+2+2L_h^2$, $\eta_q<\frac{1}{\hat{L}_q}$, and $\eta_z<\frac{1}{\hat{L}_Z}$, we have $\frac{\lambda-L_h}{2}-\frac{L_h^2}{\lambda}>0$ in \cref{lemma: descent of B and Lambda update}, $\frac{1}{2}(\frac{1}{\eta_q}-\hat{L}_{q})>0$ in \cref{lemma: descent by updating q}, and $\frac{1}{2}(\frac{1}{\eta_z}-\hat{L}_{Z})>0$ in \cref{lemma: descent of Z update}. Consequently, part (2) is a direct result of combining \cref{lemma: descent of A update}, \cref{lemma: descent of B and Lambda update}, \cref{lemma: descent by updating q}, and \cref{lemma: descent of Z update}. Therefore, we complete the proof of part (2).

Thirdly, we provide the proof of part (3). Note that $\mathcal{L}_\lambda(\mathbf{A}^{k},\mathbf{Z}^{k},\mathbf{q}^k,\mathbf{B}^k,\Lambda^k)$ is lower bounded (part (1)). Summing over $k=0,1,2,\ldots$ for the inequality in part (2), we deduce that the series $\sum_{k=0}^\infty\left(\|\mathbf{B}^{k+1}-\mathbf{B}^k\|^2+ \|\mathbf{q}^{k+1}-\mathbf{q}^{k}\|^2\right)$ converges. Consequently, we have  $\|\mathbf{B}^{k+1}-\mathbf{B}^k\|\rightarrow 0,\|\mathbf{q}^{k+1}-\mathbf{q}^k\|\rightarrow 0$ when $k\rightarrow 0$. 
Moreover, according to \cref{lemma: control dual by primal}(b), we have $\|\Lambda^{k+1}-\Lambda^k\|\rightarrow 0$. Finally, based on the update rule for $\Lambda^{k+1}$ in \eqref{alg: alternating iteration}, we obtain
\begin{align*}
\mathbf{Z}^k-\mathbf{Z}^{k+1}=\frac{1}{\lambda}(\Lambda^{k+1}-\Lambda^k)-\frac{1}{\lambda}(\Lambda^k-\Lambda^{k-1})+\mathbf{B}^k-\mathbf{B}^{k+1}.
\end{align*}
Thus, we have $\|\mathbf{Z}^k-\mathbf{Z}^{k+1}\|\rightarrow 0,k\rightarrow \infty$. Therefore, we complete the proof of part (3).

The proof is completed.
\end{proof}

\begin{lemma} 
\label{lemma: desent using lipschitz property}
If the gradient of $Q^{k+1}(\mathbf{q})$, namely $\nabla_\mathbf{q}Q^{k+1}(\mathbf{q})$, is Lipschitz continuous on $\mathcal{S}_2$ with a Lipschitz constant $\hat{L}_{q}>0$, we have the following inequality:
\begin{align*}
\ell_{\mathcal{S}_2}(\mathbf{q}^{k+1})+Q^{k+1}(\mathbf{q}^{k+1}) + \frac{1}{2}(\frac{1}{\eta_q}-\hat{L}_{q})\|\mathbf{q}^{k+1}-\mathbf{q}^{k}\|^2 \leq  \ell_{\mathcal{S}_2}(\mathbf{q}^{k})+Q^{k+1}(\mathbf{q}^{k}).
\end{align*}
Moreover, there exists $w^{k+1}\in \nabla_{\mathbf{q}} Q^{k+1}(\mathbf{q}^{k+1})+\partial \ell_{\mathcal{S}_2}(\mathbf{q}^{k+1})$ such that
\begin{align*}
\|w^{k+1}\| \leq (\frac{1}{\eta_q}+\hat{L}_q)\|\mathbf{q}^{k+1}-\mathbf{q}^k\|.
\end{align*}
\end{lemma}
\begin{proof}
 According to the update rule of $\mathbf{q}^{k+1}$ in \eqref{alg: alternating iteration} and \cref{def: prox}, we have the following inequality:
\begin{align*}
&\ell_{\mathcal{S}_2}(\mathbf{q}^{k+1}) + \frac{1}{2\eta_q}\|\mathbf{q}^{k+1}-(\mathbf{q}^k-\eta_q\nabla_q Q^{k+1}(\mathbf{q}^k))\|^2 \\
\leq &\ell_{\mathcal{S}_2}(\mathbf{q}^{k}) + \frac{1}{2\eta_q}\|\mathbf{q}^{k}-(\mathbf{q}^k-\eta_q\nabla_q Q^{k+1}(\mathbf{q}^k))\|^2.
\end{align*}
Eliminating the same terms in both sides of this inequality, we get:
\begin{align}
\label{eq: descent by proximal}
\ell_{\mathcal{S}_2}(\mathbf{q}^{k+1}) + \frac{1}{2\eta_q}\|\mathbf{q}^{k+1}-\mathbf{q}^k\|^2 + <\mathbf{q}^{k+1}-\mathbf{q}^k,\nabla_q Q^{k+1}(\mathbf{q}^k)> \leq \ell_{\mathcal{S}_2}(\mathbf{q}^{k}). 
\end{align}
Hence, we have:
\begin{align*}
&\ell_{\mathcal{S}_2}(\mathbf{q}^{k+1})+Q^{k+1}(\mathbf{q}^{k+1}) + \frac{1}{2}(\frac{1}{\eta_q}-\hat{L}_{q})\|\mathbf{q}^{k+1}-\mathbf{q}^{k}\|^2 \\
\leq &\ell_{\mathcal{S}_2}(\mathbf{q}^{k+1}) +\frac{1}{2\eta_q}\|\mathbf{q}^{k+1}-\mathbf{q}^{k}\|^2+ Q^{k+1}(\mathbf{q}^{k}) + <\nabla_\mathbf{q}Q^{k+1}(\mathbf{q}^{k}),\mathbf{q}^{k+1}-\mathbf{q}^k>\\
\leq & \ell_{\mathcal{S}_2}(\mathbf{q}^{k})+Q^{k+1}(\mathbf{q}^{k}),
\end{align*}
where the first inequality relies on the Lipschitz continuity of $\nabla_\mathbf{q} Q^{k+1}(\mathbf{q})$ and \cref{lemma: Lipschitz continuous gradient inequality}, and the second inequality relies on \eqref{eq: descent by proximal}. 
The first part is thereby proven.

According to the update rule of $\mathbf{q}^{k+1}$ in \eqref{alg: alternating iteration}, \cref{def: prox}, and the first-order optimal conditions of $\mathbf{q}^{k+1}$, it is known that there exists $v^{k+1}\in \partial \ell_{\mathcal{S}_2}(\mathbf{q}^{k+1})$ such that
\begin{align*}
v^{k+1}+\frac{1}{\eta_q}(\mathbf{q}^{k+1}-\mathbf{q}^k+\eta_q\nabla_q Q^{k+1}(\mathbf{q}^k))=0.
\end{align*}
Let $w^{k+1}=v^{k+1}+\nabla_q Q^{k+1}(\mathbf{q}^{k+1})$, then
\begin{align*}
\|w^{k+1}\| &\leq \|v^{k+1}+\nabla_q Q^{k+1}(\mathbf{q}^{k})\| + \|\nabla_q Q^{k+1}(\mathbf{q}^{k+1})-\nabla_q Q^{k+1}(\mathbf{q}^{k})\| \\
&\leq (\frac{1}{\eta_q}+\hat{L}_q)\|\mathbf{q}^{k+1}-\mathbf{q}^k\|.
\end{align*}
The latter part is thereby proven.
\end{proof}

\begin{lemma}
\label{lemma: Lipschitz continuous gradient of g wrt Z}
Fixing any $\mathbf{A}\in \mathcal{S}_0$ and any $\mathbf{q}\in \mathcal{S}_2$, $\nabla_\mathbf{Z} g(\mathbf{A},\mathbf{Z},\mathbf{q})$ is Lipschitz continuous with respect to $\mathbf{Z}$ on $\mathcal{S}_1$ with a Lipschitz constant $L_Z$.
\end{lemma}

\begin{proof}
Note that $\mathcal{S}_0$ and $\mathcal{S}_1$ are bounded sets, implying the existence of constants $C_A>0$ and $C_Z>0$ such that $\|\mathbf{A}\| \leq C_A,\forall \mathbf{A}\in \mathcal{S}_0$ and $\|\mathbf{Z}\|\leq C_Z, \forall \mathbf{Z}\in \mathcal{S}_1$. Considering the definition of $\mathcal{S}_2$, we have $\mathbf{q}\geq \epsilon_q>0, \forall \mathbf{q}\in \mathcal{S}_2$. Additionally, let $\|\mathbf{L}_\mathbf{M}\|,\|\mathbf{L}_\mathbf{C}\|\leq C_0 $ since $\mathbf{L}_\mathbf{M}$ and $\mathbf{L}_\mathbf{C}$ are given matrices. 

Define 
\begin{align*}
g_1(\mathbf{A},\mathbf{Z},\mathbf{q}):=\beta \cdot \left( \text{Tr}(\mathbf{A}(\mathbf{I}-\mathbf{D}_\mathbf{q}^{-1/2}\mathbf{ZZ}^\top\mathbf{D}_\mathbf{q}^{-1/2})\mathbf{A}^\top) + \lambda_M \text{Tr}(\mathbf{A}\mathbf{Z}_\ell\mathbf{L}_\mathbf{M}\mathbf{Z}_\ell^\top\mathbf{A}^\top)\right)
\end{align*}
and 
\begin{align*}
g_2(\mathbf{A},\mathbf{Z}):={\text{Tr}(\mathbf{A}\mathbf{Z}_\ell\mathbf{L}_\mathbf{C}\mathbf{Z}_\ell^\top\mathbf{A}^\top)}+\epsilon.
\end{align*}
Then, we obtain
\begin{align*}
g(\mathbf{A},\mathbf{Z},\mathbf{q})=\frac{g_1(\mathbf{A},\mathbf{Z},\mathbf{q})}{g_2(\mathbf{A},\mathbf{Z})}.
\end{align*}

since $\mathbf{L}_\mathbf{C}$ is a given positive semi-definite matrix, we have $g_2(\mathbf{A},\mathbf{Z})\geq \epsilon>0$.
Deriving the gradients with respect to $\mathbf{Z}$, we obtain
\begin{align*}
\nabla_\mathbf{Z}g_1(\mathbf{A},\mathbf{Z},\mathbf{q})&=2\beta \cdot \left(-\mathbf{D}_\mathbf{q}^{-1/2}\mathbf{A}^\top\mathbf{A}\mathbf{D}_\mathbf{q}^{-1/2}\mathbf{Z}+\lambda_M[\mathbf{A}^\top\mathbf{A}\mathbf{Z}_\ell\mathbf{L}_\mathbf{M},\mathbf{0}]\right)\\
\nabla_\mathbf{Z}g_2(\mathbf{A},\mathbf{Z})&=2[\mathbf{A}^\top\mathbf{A}\mathbf{Z}_\ell\mathbf{L}_\mathbf{C},\mathbf{0}].
\end{align*}
This leads to the Lipschitz continuity of $\nabla_\mathbf{Z}g_1(\mathbf{A},\mathbf{Z},\mathbf{q})$:
\begin{align*}
&\|\nabla_\mathbf{Z}g_1(\mathbf{A},\mathbf{Z}_1,\mathbf{q})-\nabla_\mathbf{Z}g_1(\mathbf{A},\mathbf{Z}_2,\mathbf{q})\| \\
&\leq 2\beta \cdot \|\left(-\mathbf{D}_\mathbf{q}^{-1/2}\mathbf{A}^\top\mathbf{A}\mathbf{D}_\mathbf{q}^{-1/2}\right)\left(\mathbf{Z}_1-\mathbf{Z}_2\right)+\lambda_M[\mathbf{A}^\top\mathbf{A}\left(\mathbf{Z}_{1_\ell}-\mathbf{Z}_{2_\ell}\right)\mathbf{L}_\mathbf{M},\mathbf{0}]\| \\
& \leq 2\beta \cdot \left(\frac{mC_A^2}{\epsilon_q}+\lambda_MC_A^2C_0\right)\|\mathbf{Z_1}-\mathbf{Z}_2\| \\
& = L_{g_1}\|\mathbf{Z_1}-\mathbf{Z}_2\|,
\end{align*}
and the Lipschitz continuity of $\nabla_\mathbf{Z}g_2(\mathbf{A},\mathbf{Z})$
\begin{align*}
\|\nabla_\mathbf{Z}g_2(\mathbf{A},\mathbf{Z}_1)-\nabla_\mathbf{Z}g_2(\mathbf{A},\mathbf{Z}_2)\| &\leq \|2[\mathbf{A}^\top\mathbf{A}\left(\mathbf{Z}_{1_\ell}-\mathbf{Z}_{2_\ell}\right)\mathbf{L}_\mathbf{C},\mathbf{0}]\|
 \\
& \leq 2C_A^2C_0\|\mathbf{Z}_1-\mathbf{Z}_2\| \\
&= L_{g_2}\|\mathbf{Z}_1-\mathbf{Z}_2\|,
\end{align*}
where $L_{g_1}:=2\beta \cdot \left(\frac{mC_A^2}{\epsilon_q}+\lambda_MC_A^2C_0\right)$ and $L_{g_2}:=2C_A^2C_0$. 

Similarly, from the boundedness of $\mathbf{A}$, $\mathbf{Z}$, and $\mathbf{q}$, we have:
\begin{align}
\label{ineq: boundedness of gradients}
&\|\nabla_\mathbf{Z} {g_1}(\mathbf{A},\mathbf{Z},\mathbf{q})\| \leq C_ZL_{g_1}, \nonumber\\
&\|\nabla_\mathbf{Z} {g_2}(\mathbf{A},\mathbf{Z})\| \leq C_ZL_{g_2}, \nonumber\\
&\|g_1(\mathbf{A},\mathbf{Z},\mathbf{q})\| \leq \beta \cdot \left((\sqrt{m}+\frac{mC_Z^2}{\epsilon_q})C_A^2+\lambda_M C_0C_A^2C_Z^2\right)=:M_{g_1},\nonumber\\
&\|g_2(\mathbf{A},\mathbf{Z})\| \leq C_0C_A^2C_Z^2+\epsilon=:M_{g_2}.
\end{align}
As $g(\mathbf{A},\mathbf{Z},\mathbf{q})=\frac{g_1(\mathbf{A},\mathbf{Z},\mathbf{q})}{g_2(\mathbf{A},\mathbf{Z})}$, the gradient is given by:
\begin{align}
\label{eq: gradient of g}
\nabla_Z g(\mathbf{A},\mathbf{Z},\mathbf{q}) = \frac{\nabla_Z g_1(\mathbf{A},\mathbf{Z},\mathbf{q})}{g_2(\mathbf{A},\mathbf{Z})} -  \frac{g_1(\mathbf{A},\mathbf{Z},\mathbf{q}) \cdot \nabla_Z g_2(\mathbf{A},\mathbf{Z})}{g_2(\mathbf{A},\mathbf{Z})\cdot g_2(\mathbf{A},\mathbf{Z})}.
\end{align}
We now derive the Lipschitz continuity of the first component in \eqref{eq: gradient of g} as follows:
\begin{align}
\label{ineq:g regularity1}
&\|\frac{\nabla_Z g_1(\mathbf{A},\mathbf{Z}_1,\mathbf{q})}{g_2(\mathbf{A},\mathbf{Z}_1)} - \frac{\nabla_Z g_1(\mathbf{A},\mathbf{Z}_2,\mathbf{q})}{g_2(\mathbf{A},\mathbf{Z}_2)}\| \nonumber\\
 \leq &\|\frac{\nabla_Z g_1(\mathbf{A},\mathbf{Z}_1,\mathbf{q})}{g_2(\mathbf{A},\mathbf{Z}_1)} -\frac{\nabla_Z g_1(\mathbf{A},\mathbf{Z}_2,\mathbf{q})}{g_2(\mathbf{A},\mathbf{Z}_1)}\| +\|\frac{\nabla_Z g_1(\mathbf{A},\mathbf{Z}_2,\mathbf{q})}{g_2(\mathbf{A},\mathbf{Z}_1)} - \frac{\nabla_Z g_1(\mathbf{A},\mathbf{Z}_2,\mathbf{q})}{g_2(\mathbf{A},\mathbf{Z}_2)} \| \nonumber\\
 \leq &\frac{\|\nabla_\mathbf{Z}g_1(\mathbf{A},\mathbf{Z}_1,\mathbf{q})-\nabla_\mathbf{Z}g_1(\mathbf{A},\mathbf{Z}_2,\mathbf{q})\|}{\|g_2(\mathbf{A},\mathbf{Z}_1)\|} \nonumber\\
&+ \frac{\|\nabla g_1(\mathbf{A},\mathbf{Z}_2,\mathbf{q})\|}{\|g_2(\mathbf{A},\mathbf{Z}_1)g_2(\mathbf{A},\mathbf{Z}_2)\|} \cdot \|g_2(\mathbf{A},\mathbf{Z}_1)-g_2(\mathbf{A},\mathbf{Z}_2)\| \nonumber\\
\leq &\frac{L_{g_1}}{\epsilon}\|\mathbf{Z}_1-\mathbf{Z}_2\| + \frac{C_ZL_{g_1}}{\epsilon^2} \|(\nabla_\mathbf{Z} g_2(\mathbf{A},\theta \mathbf{Z}_1+(1-\theta)\mathbf{Z}_2),\mathbf{Z}_1-\mathbf{Z}_2)\| \nonumber\\
\leq &\frac{L_{g_1}}{\epsilon}\|\mathbf{Z}_1-\mathbf{Z}_2\| + \frac{C_Z^2L_{g_1}L_{g_2}}{\epsilon^2}\|\mathbf{Z}_1-\mathbf{Z}_2\|,
\end{align}
where the third inequality relies on the Mean Value Theorem and $\theta \in (0,1)$.

Subsequently, we derive the Lipschitz continuity of the second component in \eqref{eq: gradient of g}. 
First, using the triangle inequality, Lipschitz continuity of $\nabla_\mathbf{Z} g_2(\mathbf{A},\mathbf{Z})$, and the inequalities in \eqref{ineq: boundedness of gradients}, we obtain
\begin{align}
\label{ineq: ineq:g regularity2 part 1}
&\|g_1(\mathbf{A},\mathbf{Z}_1,\mathbf{q})\cdot \nabla_\mathbf{Z} g_2(\mathbf{A},\mathbf{Z}_1)-g_1(\mathbf{A},\mathbf{Z}_2,\mathbf{q})\cdot \nabla_\mathbf{Z} g_2(\mathbf{A},\mathbf{Z}_2)\| \nonumber \\
 \leq &\|g_1(\mathbf{A},\mathbf{Z}_1,\mathbf{q})\cdot \nabla_\mathbf{Z} g_2(\mathbf{A},\mathbf{Z}_1)-g_1(\mathbf{A},\mathbf{Z}_1,\mathbf{q})\cdot \nabla_\mathbf{Z} g_2(\mathbf{A},\mathbf{Z}_2)\| \nonumber\\
&+ \|g_1(\mathbf{A},\mathbf{Z}_1,\mathbf{q})\cdot \nabla_\mathbf{Z} g_2(\mathbf{A},\mathbf{Z}_2)-g_1(\mathbf{A},\mathbf{Z}_2,\mathbf{q})\cdot \nabla_\mathbf{Z} g_2(\mathbf{A},\mathbf{Z}_2)\| \nonumber\\
 \leq &M_{g_1}L_{g_2}\|\mathbf{Z}_1-\mathbf{Z}_2\| + C_ZL_{g_2}\|(\nabla g_1(\mathbf{A},\theta\mathbf{Z}_1+(1-\theta)\mathbf{Z}_2,\mathbf{q}),\mathbf{Z}_1-\mathbf{Z}_2)\| \nonumber\\
 \leq &M_{g_1}L_{g_2}\|\mathbf{Z}_1-\mathbf{Z}_2\| + C_Z^2L_{g_1}L_{g_2}\|\mathbf{Z}_1-\mathbf{Z}_2\|
\end{align}
and 
\begin{align}
\label{ineq: ineq:g regularity2 part 2}
&\|g_2(\mathbf{A},\mathbf{Z}_1) \cdot g_2(\mathbf{A},\mathbf{Z}_1) - g_2(\mathbf{A},\mathbf{Z}_2) \cdot g_2(\mathbf{A},\mathbf{Z}_2)\| \nonumber\\
& \leq \|(2g_2(\mathbf{A},\theta \mathbf{Z}_1+(1-\theta)\mathbf{Z}_2)\nabla g_2(\mathbf{A},\theta \mathbf{Z}_1+(1-\theta)\mathbf{Z}_2),\mathbf{Z}_1-\mathbf{Z}_2)\| \nonumber\\
& \leq 2C_ZL_{g_2}M_{g_2}\|\mathbf{Z}_1-\mathbf{Z}_2\|.
\end{align}
Consequently, combining \eqref{ineq: ineq:g regularity2 part 1} and \eqref{ineq: ineq:g regularity2 part 2}, we obtain Lipschitz continuity of the second component in \eqref{eq: gradient of g} as follows:
\begin{align}
\label{ineq:g regularity2}
&\|\frac{g_1(\mathbf{A},\mathbf{Z}_1,\mathbf{q}) \cdot \nabla g_2(\mathbf{A},\mathbf{Z}_1)}{g_2(\mathbf{A},\mathbf{Z}_1)\cdot g_2(\mathbf{A},\mathbf{Z}_1)} - \frac{g_1(\mathbf{A},\mathbf{Z}_2,\mathbf{q}) \cdot \nabla g_2(\mathbf{A},\mathbf{Z}_2)}{g_2(\mathbf{A},\mathbf{Z}_2)\cdot g_2(\mathbf{A},\mathbf{Z}_2)}\| \nonumber\\
\leq &\|\frac{g_1(\mathbf{A},\mathbf{Z}_1,\mathbf{q}) \cdot \nabla g_2(\mathbf{A},\mathbf{Z}_1)}{g_2(\mathbf{A},\mathbf{Z}_1)\cdot g_2(\mathbf{A},\mathbf{Z}_1)} - \frac{g_1(\mathbf{A},\mathbf{Z}_2,\mathbf{q}) \cdot \nabla g_2(\mathbf{A},\mathbf{Z}_2)}{g_2(\mathbf{A},\mathbf{Z}_1)\cdot g_2(\mathbf{A},\mathbf{Z}_1)}\| \nonumber\\
&+ \|\frac{g_1(\mathbf{A},\mathbf{Z}_2,\mathbf{q}) \cdot \nabla g_2(\mathbf{A},\mathbf{Z}_2)}{g_2(\mathbf{A},\mathbf{Z}_1)\cdot g_2(\mathbf{A},\mathbf{Z}_1)} - \frac{g_1(\mathbf{A},\mathbf{Z}_2,\mathbf{q}) \cdot \nabla g_2(\mathbf{A},\mathbf{Z}_2)}{g_2(\mathbf{A},\mathbf{Z}_2)\cdot g_2(\mathbf{A},\mathbf{Z}_2)}\| \nonumber\\
\leq &\frac{M_{g_1}L_{g_2}+C_Z^2L_{g_1}L_{g_2}}{\epsilon^2} \|\mathbf{Z}_1-\mathbf{Z}_2\| + C_ZL_{g_2}M_{g_1}\|\frac{g_2(\mathbf{A},\mathbf{Z}_1)^2-g_2(\mathbf{A},\mathbf{Z}_2)^2}{g_2(\mathbf{A},\mathbf{Z}_1)^2g_2(\mathbf{A},\mathbf{Z}_2)^2}\| \nonumber\\
\leq &\frac{M_{g_1}L_{g_2}+C_Z^2L_{g_1}L_{g_2}}{\epsilon^2} \|\mathbf{Z}_1-\mathbf{Z}_2\| + \frac{2C_Z^2L^2_{g_2}M_{g_1}M_{g_2}}{\epsilon^4}\|\mathbf{Z}_1-\mathbf{Z}_2\|
\end{align}
Combine \eqref{ineq:g regularity1} and \eqref{ineq:g regularity2}, we obtain the Lipschitz continuity of $\nabla_{\mathbf{Z}} g(\mathbf{A},\mathbf{Z},\mathbf{q})$:
\begin{align*}
&\|\nabla_\mathbf{Z} g(\mathbf{A},\mathbf{Z}_1,\mathbf{q})-\nabla_\mathbf{Z} g(\mathbf{A},\mathbf{Z}_2,\mathbf{q})\| \\
\leq &\left(\frac{L_{g_1}}{\epsilon}+\frac{2C_Z^2L_{g_1}L_{g_2}+M_{g_1}L_{g_2}}{\epsilon^2}+\frac{2C_Z^2L^2_{g_2}M_{g_1}M_{g_2}}{\epsilon^4}\right)\|\mathbf{Z}_1-\mathbf{Z}_2\|,
\end{align*}
which completes the proof.
\end{proof}

\begin{lemma} 
\label{lemma: Lipschitz continuous gradient of g0 wrt Z}
For any $\mathbf{q}\in \mathcal{S}_2$, $\nabla_\mathbf{Z} g_0(\mathbf{Z},\mathbf{q})$ is Lipschitz continuous with respect to $\mathbf{Z}$ on $\mathcal{S}_1$ with a Lipschitz constant $C_{g_0}>0$.
\end{lemma}

\begin{proof}
Since $
g_0(\mathbf{Z},\mathbf{q}) := \frac{\lambda_0}{2}\|\mathbf{q}-\mathbf{ZZ}^\top\mathbf{1}_m\|^2$, we can obtain the gradient of $g_0$ with respect to $\mathbf{Z}$ as follows:
\begin{align*}
\nabla_\mathbf{Z} g_0(\mathbf{Z},\mathbf{q}) = \lambda_0\cdot \left(\mathbf{1}_m(\mathbf{ZZ}^\top\mathbf{1}_m-\mathbf{q})^\top+(\mathbf{ZZ}^\top\mathbf{1}_m-\mathbf{q})\mathbf{1}_m^\top\right)\mathbf{Z}. 
\end{align*}
Note that $\mathcal{S}_1$ is a bounded set, implying the existence of constant $C_Z>0$ such that $\|\mathbf{Z}\|\leq C_Z, \forall \mathbf{Z}\in \mathcal{S}_1$. Considering the definition of $\mathcal{S}_2$, we have $C_q \geq \mathbf{q}\geq \epsilon_q>0, \forall \mathbf{q}\in \mathcal{S}_2$. Then, using the triangle inequality and the boundedness of $\mathbf{Z}$ and $\mathbf{q}$, we obtain the following two inequalities for any $\mathbf{Z}_1,\mathbf{Z}_2\in \mathcal{S}_1$: 
\begin{align*}
&\|\mathbf{1}_m(\mathbf{Z}_1\mathbf{Z}_1^\top\mathbf{1}_m-\mathbf{q})^\top - \mathbf{1}_m(\mathbf{Z}_2\mathbf{Z}_2^\top\mathbf{1}_m-\mathbf{q})^\top\| \\
\leq &\|\mathbf{1}_m(\mathbf{Z}_1\mathbf{Z}_1^\top\mathbf{1}_m-\mathbf{q})^\top - \mathbf{1}_m(\mathbf{Z}_2\mathbf{Z}_1^\top\mathbf{1}_m-\mathbf{q})^\top\| \\
&+\|\mathbf{1}_m(\mathbf{Z}_2\mathbf{Z}_1^\top\mathbf{1}_m-\mathbf{q})^\top - \mathbf{1}_m(\mathbf{Z}_2\mathbf{Z}_2^\top\mathbf{1}_m-\mathbf{q})^\top\| \\
\leq & 2mC_Z\|\mathbf{Z}_1-\mathbf{Z}_2\|,
\end{align*}
and
\begin{align*}
&\|\mathbf{1}_m(\mathbf{Z}_1\mathbf{Z}_1^\top\mathbf{1}_m-\mathbf{q})^\top\mathbf{Z}_1 - \mathbf{1}_m(\mathbf{Z}_2\mathbf{Z}_2^\top\mathbf{1}_m-\mathbf{q})^\top\mathbf{Z}_2\| \\ \leq 
& \|\mathbf{1}_m(\mathbf{Z}_1\mathbf{Z}_1^\top\mathbf{1}_m-\mathbf{q})^\top\mathbf{Z}_1 - \mathbf{1}_m(\mathbf{Z}_2\mathbf{Z}_2^\top\mathbf{1}_m-\mathbf{q})^\top\mathbf{Z}_1\|\\ &+\|\mathbf{1}_m(\mathbf{Z}_2\mathbf{Z}_2^\top\mathbf{1}_m-\mathbf{q})^\top\mathbf{Z}_1 - \mathbf{1}_m(\mathbf{Z}_2\mathbf{Z}_2^\top\mathbf{1}_m-\mathbf{q})^\top\mathbf{Z}_2\| \\
\leq  &(2mC_Z^2+mC_Z^2+mC_q)\|\mathbf{Z}_1-\mathbf{Z}_2\|. 
\end{align*}
Combining the above two inequalities, we obtain
\begin{align*}
\|\nabla_\mathbf{Z} g_0(\mathbf{Z}_1,\mathbf{q})-\nabla_\mathbf{Z} g_0(\mathbf{Z}_2,\mathbf{q})\|\leq 2\lambda_0\cdot (2mC_Z^2+mC_Z^2+mC_q)\|\mathbf{Z}_1-\mathbf{Z}_2\|.
\end{align*}
Define $C_{g_0}$ as $C_{g_0}:=2\lambda_0\cdot (2mC_Z^2+mC_Z^2+mC_q)$, and we complete the proof.
\end{proof}

\begin{lemma}
\label{lemma: Lipschitz continuous gradient of g wrt q}
Fixing any $\mathbf{A}\in \mathcal{S}_0$ and any $\mathbf{Z}\in \mathcal{S}_1$,  $\nabla_\mathbf{q} g(\mathbf{A},\mathbf{Z},\mathbf{q})$ is Lipschitz continuous with respect to $\mathbf{q}$ on $\mathcal{S}_2$ with a Lipschitz constant $L_q$.
\end{lemma}

\begin{proof}
Note that $\mathcal{S}_0$ and $\mathcal{S}_1$ are bounded sets, implying the existence of constants $C_A>0$ and $C_Z>0$ such that $\|\mathbf{A}\| \leq C_A,\forall \mathbf{A}\in \mathcal{S}_0$ and $\|\mathbf{Z}\|\leq C_Z, \forall \mathbf{Z}\in \mathcal{S}_1$. Considering the definition of $\mathcal{S}_2$, we have $C_q \geq \mathbf{q}\geq \epsilon_q>0, \forall \mathbf{q}\in \mathcal{S}_2$. Additionally, let $\|\mathbf{L}_\mathbf{M}\|,\|\mathbf{L}_\mathbf{C}\|\leq C_0 $ since $\mathbf{L}_\mathbf{M}$ and $\mathbf{L}_\mathbf{C}$ are given matrices. 

Let $g_2(\mathbf{A},\mathbf{Z}):={\text{Tr}(\mathbf{A}\mathbf{Z}_\ell\mathbf{L}_\mathbf{C}\mathbf{Z}_\ell^\top\mathbf{A}^\top)}+\epsilon$. since $\mathbf{L}_\mathbf{C}$ is a given positive semi-definite matrix, we have $g_2(\mathbf{A},\mathbf{Z})\geq \epsilon>0$.
Then, the gradient with respect to $\mathbf{q}$ is given by:
\begin{align*}
\nabla_\mathbf{q}g(\mathbf{A},\mathbf{Z},\mathbf{q}) = \frac{\beta}{g_2(\mathbf{A},\mathbf{Z})}diag(\mathbf{ZZ}^\top diag(\mathbf{q}^{-1/2})\mathbf{A}^\top\mathbf{A}diag(\mathbf{q}^{-3/2}))
\end{align*}
Using the triangle inequality and the boundedness of $\mathbf{A}$, $\mathbf{Z}$, and $\mathbf{q}$, we have 
\begin{align}
\label{eq: q lipschitz differentiable}
&\|\nabla_\mathbf{q}g(\mathbf{A},\mathbf{Z},\mathbf{q}_1)-\nabla_\mathbf{q}g(\mathbf{A},\mathbf{Z},\mathbf{q}_2)\| \nonumber \\
=&\|\frac{\beta}{g_2(\mathbf{A},\mathbf{Z})}diag(\mathbf{ZZ}^\top diag(\mathbf{q}_1^{-1/2})\mathbf{A}^\top\mathbf{A}diag(\mathbf{q}_1^{-3/2})) \nonumber\\
&- \frac{\beta}{g_2(\mathbf{A},\mathbf{Z})}diag(\mathbf{ZZ}^\top diag(\mathbf{q}_2^{-1/2})\mathbf{A}^\top\mathbf{A}diag(\mathbf{q}_2^{-3/2}))\| \nonumber\\
\leq & \|\frac{\beta}{g_2(\mathbf{A},\mathbf{Z})}diag(\mathbf{ZZ}^\top diag(\mathbf{q}_1^{-1/2})\mathbf{A}^\top\mathbf{A}diag(\mathbf{q}_1^{-3/2})) \nonumber\\
&- \frac{\beta}{g_2(\mathbf{A},\mathbf{Z})}diag(\mathbf{ZZ}^\top diag(\mathbf{q}_1^{-1/2})\mathbf{A}^\top\mathbf{A}diag(\mathbf{q}_2^{-3/2}))\| \\
&+ \|\frac{\beta}{g_2(\mathbf{A},\mathbf{Z})}diag(\mathbf{ZZ}^\top diag(\mathbf{q}_1^{-1/2})\mathbf{A}^\top\mathbf{A}diag(\mathbf{q}_2^{-3/2})) \nonumber\\
&- \frac{\beta}{g_2(\mathbf{A},\mathbf{Z})}diag(\mathbf{ZZ}^\top diag(\mathbf{q}_2^{-1/2})\mathbf{A}^\top\mathbf{A}diag(\mathbf{q}_2^{-3/2}))\| \nonumber \\
\leq & \frac{\beta C_A^2C_Z^2}{\epsilon}\|\mathbf{q}_1^{-1/2}\|\|\mathbf{q}_1^{-3/2}-\mathbf{q}_2^{-3/2}\| + \frac{\beta C_A^2C_Z^2}{\epsilon}\|\mathbf{q}_2^{-3/2}\|\|\mathbf{q}_1^{-1/2}-\mathbf{q}_2^{-1/2}\|. \nonumber
\end{align}
In the above derivation, the first inequality utilizes the triangle inequality, while the second inequality employs the relations $\|diag(\mathbf{C})\|\leq \|\mathbf{C}\|, \forall \text{~square~matrix~} \mathbf{C}$, $\|diag(\mathbf{q})\| = \|\mathbf{q}\|, \forall \text{~vector~} \mathbf{q}$, and $g_2(\mathbf{A},\mathbf{Z})\geq \epsilon$. Additionally, due to $\mathbf{q}_1 \geq \epsilon_q,\mathbf{q}_2 \geq \epsilon_q$, we have 
\begin{align}
\label{ineq: q1q2}
\|\mathbf{q}_1^{-1/2}\|\leq \sqrt{\frac{m}{\epsilon_q}},~~~~\|\mathbf{q}_2^{-3/2}\|\leq \sqrt{\frac{m}{\epsilon_q^3}}.
\end{align}
Let $q_1(i)$ and $q_2(i)$ denote the $i$-th elements of vectors 
 $\mathbf{q}_1$ and $\mathbf{q}_2$, respectively. Then,
\begin{align*}
|q_1(i)^{-3/2}-q_2(i)^{-3/2}| &=|\frac{q_1(i)\sqrt{q_1(i)}-q_2(i)\sqrt{q_2(i)}}{q_1(i)\sqrt{q_1(i)}q_2(i)\sqrt{q_2(i)}}| \\
&=|\frac{q_1(i)^3-q_2(i)^3}{q_1(i)\sqrt{q_1(i)}q_2(i)\sqrt{q_2(i)}(q_1(i)\sqrt{q_1(i)}+q_2(i)\sqrt{q_2(i)})}| \\
& = |\frac{(q_1(i)^2+q_1(i)q_2(i)+q_2(i)^2)(q_1(i)-q_2(i))}{q_1(i)\sqrt{q_1(i)}q_2(i)\sqrt{q_2(i)}(q_1(i)\sqrt{q_1(i)}+q_2(i)\sqrt{q_2(i)})}|   \\
&\leq \frac{3C_q^2}{2\epsilon_q^{9/2}}|q_1(i)-q_2(i)|.
\end{align*}
Thus,
\begin{align}
\label{ineq: q1-q2-3-2}
\|\mathbf{q}_1^{-3/2}-\mathbf{q}_2^{-3/2}\| =\sqrt{\sum_{i=1}^m|q_1(i)^{-3/2}-q_2(i)^{-3/2}|^2}\leq \frac{3C_q^2}{2\epsilon_q^{9/2}}\|\mathbf{q}_1-\mathbf{q}_2\|.
\end{align}
Similarly, we have
\begin{align*}
|q_1(i)^{-1/2}-q_2(i)^{-1/2}| = |\frac{\sqrt{q_1(i)}-\sqrt{q_2(i)}}{\sqrt{q_1(i)q_2(i)}}| &= |\frac{q_1(i)-q_2(i)}{\sqrt{q_1(i)q_2(i)}(\sqrt{q_1(i)}+\sqrt{q_2(i)})}| \\
&\leq \frac{1}{2\epsilon_q^{3/2}}|q_1(i)-q_2(i)|.
\end{align*}
Hence,
\begin{align}
\label{ineq: q1-q2-1-2}
\|\mathbf{q}_1^{-1/2}-\mathbf{q}_2^{-1/2}\| \leq \frac{1}{2\epsilon_q^{3/2}}\|\mathbf{q}_1-\mathbf{q}_2\|
\end{align}
Combining \eqref{eq: q lipschitz differentiable}, \eqref{ineq: q1q2}, \eqref{ineq: q1-q2-3-2}, and \eqref{ineq: q1-q2-1-2}, we obtain
\begin{align*}
\|\nabla_\mathbf{q}g(\mathbf{A},\mathbf{Z},\mathbf{q}_1)-\nabla_\mathbf{q}g(\mathbf{A},\mathbf{Z},\mathbf{q}_2)\| \leq \frac{\beta C_A^2 C_Z^2}{\epsilon}\left(\frac{3\sqrt{m}C_q^2}{2\epsilon_q^{5}}+\frac{\sqrt{m}}{2\epsilon_q^{3}}\right)\|\mathbf{q}_1-\mathbf{q}_2\|,
\end{align*}
which completes the proof.
\end{proof}

\begin{lemma} 
\label{lemma: Lipschitz continuous gradient inequality}
Given a function $h: C\subset \mathbb{R}^n \rightarrow \mathbb{R}$, if the gradient of $h$, namely $\nabla h$, is Lipschitz continuous on the convex set $C$ with a Lipschitz constant $L_h>0$, the inequality below is satisfied for any $x,y\in C$:
\begin{align*}
|h(y)-h(x)-<\nabla h(x),y-x>| \leq \frac{L_h}{2}\|y-x\|^2.
\end{align*}
\end{lemma}

\begin{proof}
Using the Lipschitz continuity assumption of $\nabla h$, we obtain
\begin{align*}
|h(y)-h(x)-<\nabla h(x),y-x>| &= |\int_{0}^1 (\nabla h(x+t(y-x))-\nabla h(x))^\top (y-x) \mathrm{d}t| \\
& \leq \int_{0}^1 |(\nabla h(x+t(y-x))-\nabla h(x))^\top (y-x)| \mathrm{d}t\\
& \leq \int_{0}^1 \|(\nabla h(x+t(y-x))-\nabla h(x))^\top\| \cdot \|(y-x)\| \mathrm{d}t \\
& \leq L_h \|y-x\|^2 \int_{0}^1 t \mathrm{d}t \\
& = \frac{L_h}{2} \|y-x\|^2.
\end{align*}
This completes the proof.
\end{proof}

\begin{lemma}
\label{lemma: Lipschitz continuous gradient of h}
$\nabla h(\mathbf{B})$ is Lipschitz continuous with a Lipschitz constant $L_h>0$.
\end{lemma}

\begin{proof}
Using the fact that $\nabla h(B)=\sum_{v=1}^V \mathbf{U}^{(v)^\top}(\mathbf{U}^{(v)}\mathbf{B}-\mathbf{X}^{(v)})$ and the definition of Lipschitz continuity, we obtain the Lipschitz continuity of $\nabla h(\mathbf{B})$.
\end{proof}

\begin{lemma} 
\label{lemma: control dual by primal}
For all $k\in \mathbb{N}$, it holds that 

\textit{(a)} $\Lambda^k=-\nabla h(\mathbf{B}^k)$

\textit{(b)} $\|\Lambda^{k+1}-\Lambda^{k}\| \leq L_h\|\mathbf{B}^k-\mathbf{B}^{k+1}\| $
\end{lemma}

\begin{proof}
Part (a) is derived using the first-order optimal condition of $\mathbf{B}^k$ in \eqref{alg: alternating iteration}: $0=\nabla h(\mathbf{B}^k)+\Lambda^{k-1}+\lambda (\mathbf{B}^k- \mathbf{Z}^k)$ and the update rule of $\Lambda^{k}$ in \eqref{alg: alternating iteration}: $\Lambda^{k}=\Lambda^{k-1}+\lambda(\mathbf{B}^k-\mathbf{Z}^k)$. 
Part (b) is derived using part (a) and the Lipschitz continuity of $\nabla h(\mathbf{B})$.
\end{proof}

\begin{lemma}[descent of $\mathcal{L}_{\lambda}$ during $\mathbf{A}$ update]
\label{lemma: descent of A update}
When updating $\mathbf{A}^{k+1}$ in \eqref{alg: alternating iteration}, we have
\begin{align*}
\mathcal{L}_\lambda(\mathbf{A}^{k},\mathbf{Z}^k,\mathbf{q}^k,\mathbf{B}^k,\Lambda^k) \geq \mathcal{L}_\lambda(\mathbf{A}^{k+1},\mathbf{Z}^{k},\mathbf{q}^k,\mathbf{B}^{k},\Lambda^k).
\end{align*}
\end{lemma}

\begin{proof}
The inequality could be directly obtained using the definition of $\mathbf{A}^{k+1}$ in \eqref{alg: alternating iteration}.
\end{proof}

\begin{lemma}[descent of $\mathcal{L}_{\lambda}$ during $\mathbf{B}$ and $\Lambda$ update] 
\label{lemma: descent of B and Lambda update}
When updating $\mathbf{B}^{k+1}$ and $\Lambda^{k+1}$ in \eqref{alg: alternating iteration}, we have
\begin{align*}
&\mathcal{L}_\lambda(\mathbf{A}^{k+1},\mathbf{Z}^{k+1},\mathbf{q}^{k+1},\mathbf{B}^k,\Lambda^k) - \mathcal{L}_\lambda(\mathbf{A}^{k+1},\mathbf{Z}^{k+1},\mathbf{q}^{k+1},\mathbf{B}^{k+1},\Lambda^{k+1}) \\
&\geq (\frac{\lambda-L_h}{2}-\frac{L_h^2}{\lambda})\|\mathbf{B}^k-\mathbf{B}^{k+1}\|^2.
\end{align*}
\end{lemma}

\begin{proof} We derive that 
\begin{align*}
&\mathcal{L}_\lambda(\mathbf{A}^{k+1},\mathbf{Z}^{k+1},\mathbf{q}^{k+1},\mathbf{B}^k,\Lambda^k) - \mathcal{L}_\lambda(\mathbf{A}^{k+1},\mathbf{Z}^{k+1},\mathbf{q}^{k+1},\mathbf{B}^{k+1},\Lambda^{k+1}) \\
=& h(\mathbf{B}^k)-h(\mathbf{B}^{k+1}) + <\Lambda^k,\mathbf{B}^k-\mathbf{Z}^{k+1}> -<\Lambda^{k+1},\mathbf{B}^{k+1}-\mathbf{Z}^{k+1}>  \\
&+ \frac{\lambda}{2}\|\mathbf{B}^k-\mathbf{Z}^{k+1}\|^2 - \frac{\lambda}{2}\||\mathbf{B}^{k+1}-\mathbf{Z}^{k+1}\|^2 \\
= & h(\mathbf{B}^k)-h(\mathbf{B}^{k+1}) + \frac{\lambda}{2}\||\mathbf{B}^k-\mathbf{B}^{k+1}\|^2 - 
\frac{1}{\lambda}\|\Lambda^{k+1}-\Lambda^k\|^2 \\
&+ <\Lambda^{k+1},\mathbf{B}^{k}-\mathbf{B}^{k+1}>\\
=& h(\mathbf{B}^k)-h(\mathbf{B}^{k+1}) - <\nabla h(\mathbf{B}^{k+1}),\mathbf{B}^{k}-\mathbf{B}^{k+1}> \\
&+ \frac{\lambda}{2}\||\mathbf{B}^k-\mathbf{B}^{k+1}\|^2 - 
\frac{1}{\lambda}\|\Lambda^{k+1}-\Lambda^k\|^2 \\
\geq & (\frac{\lambda-L_h}{2}-\frac{L_h^2}{\lambda})\|\mathbf{B}^k-\mathbf{B}^{k+1}\|^2, 
\end{align*}
where the second equality relies on the update rule of $\Lambda^{k+1}$ in \eqref{alg: alternating iteration}: $\Lambda^{k+1}=\Lambda^{k}+\lambda(\mathbf{B}^{k+1}-\mathbf{Z}^{k+1})$, and the third equality is based on \cref{lemma: control dual by primal}(a): $\Lambda^k=-\nabla h(\mathbf{B}^k)$. The first inequality utilizes \cref{lemma: control dual by primal}(b), \cref{lemma: Lipschitz continuous gradient inequality}, and the Lipschitz continuity property of $\nabla h$ described in \cref{lemma: Lipschitz continuous gradient of h}.
\end{proof}

\begin{lemma}[sufficient descent of $\mathcal{L}_{\lambda}$ during $\mathbf{q}$ update]
\label{lemma: descent by updating q}
When updating $\mathbf{q}^{k+1}$ in \eqref{alg: alternating iteration}, we obtain 
\begin{align*}
\mathcal{L}_\lambda(\mathbf{A}^{k+1},\mathbf{Z}^{k+1},\mathbf{q}^k,\mathbf{B}^k,\Lambda^k)-\mathcal{L}_\lambda(\mathbf{A}^{k+1},\mathbf{Z}^{k+1},\mathbf{q}^{k+1},\mathbf{B}^k,\Lambda^k) \geq \frac{1}{2}(\frac{1}{\eta_q}-\hat{L}_{q})\|\mathbf{q}^{k+1}-\mathbf{q}^{k}\|^2
\end{align*}
\end{lemma}

\begin{proof}
According to \cref{lemma: main lemma}, we have $\nabla_\mathbf{q} Q^{k+1}(\mathbf{q})$ is Lipschitz continuous on $\mathcal{S}_2$ with a Lipschitz constant $\hat{L}_q>0$. Then, applying \cref{lemma: desent using lipschitz property}, we obtain
\begin{align}
\label{eq: descent of updating q}
\ell_{\mathcal{S}_2}(\mathbf{q}^{k+1})+Q^{k+1}(\mathbf{q}^{k+1}) + \frac{1}{2}(\frac{1}{\eta_q}-\hat{L}_{q})\|\mathbf{q}^{k+1}-\mathbf{q}^{k}\|^2 \leq  \ell_{\mathcal{S}_2}(\mathbf{q}^{k})+Q^{k+1}(\mathbf{q}^{k})
\end{align}
Therefore, using \eqref{eq: descent of updating q} and the definition of $Q^{k+1}(\mathbf{q})$ in \eqref{eq:Q proof}, we directly obtain
\begin{align*}
&\mathcal{L}_\lambda(\mathbf{A}^{k+1},\mathbf{Z}^{k+1},\mathbf{q}^k,\mathbf{B}^k,\Lambda^k)-\mathcal{L}_\lambda(\mathbf{A}^{k+1},\mathbf{Z}^{k+1},\mathbf{q}^{k+1},\mathbf{B}^k,\Lambda^k) \\
= & \ell_{\mathcal{S}_2}(\mathbf{q}^{k})+Q^{k+1}(\mathbf{q}^{k}) - \ell_{\mathcal{S}_2}(\mathbf{q}^{k+1}) - Q^{k+1}(\mathbf{q}^{k+1}) \\
\geq &\frac{1}{2}(\frac{1}{\eta_q}-\hat{L}_{q})\|\mathbf{q}^{k+1}-\mathbf{q}^{k}\|^2,
\end{align*}
which completes the proof.
\end{proof}

\begin{lemma}[sufficient descent of $\mathcal{L}_{\lambda}$ during $\mathbf{Z}$ update]
\label{lemma: descent of Z update}
When updating $\mathbf{Z}^{k+1}$ in \eqref{alg: alternating iteration}, we have 
\begin{align*}
\mathcal{L}_\lambda(\mathbf{A}^{k+1},\mathbf{Z}^{k},\mathbf{q}^{k},\mathbf{B}^k,\Lambda^k)-\mathcal{L}_\lambda(\mathbf{A}^{k+1},\mathbf{Z}^{k+1},\mathbf{q}^{k},\mathbf{B}^k,\Lambda^k) \geq \frac{1}{2}(\frac{1}{\eta_z}-\hat{L}_{Z})\|\mathbf{Z}^{k+1}-\mathbf{Z}^{k}\|^2.
\end{align*}
\end{lemma}

\begin{proof}
According to \cref{lemma: main lemma}, we have $\nabla_\mathbf{Z} P^{k+1}(\mathbf{Z})$ is Lipschitz continuous on $\mathcal{S}_1$ with a Lipschitz constant $\hat{L}_Z>0$. Then, similar to the proof in \cref{lemma: desent using lipschitz property}, we obtain
\begin{align}
\label{eq: descent of updating Z}
\ell_{\mathcal{S}_1}(\mathbf{Z}^{k+1})+P^{k+1}(\mathbf{Z}^{k+1}) + \frac{1}{2}(\frac{1}{\eta_z}-\hat{L}_{Z})\|\mathbf{Z}^{k+1}-\mathbf{Z}^{k}\|^2 \leq  \ell_{\mathcal{S}_1}(\mathbf{Z}^{k})+P^{k+1}(\mathbf{Z}^{k}).
\end{align}
Therefore, using \eqref{eq: descent of updating Z} and the definition of $P^{k+1}(\mathbf{Z})$ in \eqref{eq:P proof}, we directly obtain
\begin{align*}
&\mathcal{L}_\lambda(\mathbf{A}^{k+1},\mathbf{Z}^{k},\mathbf{q}^k,\mathbf{B}^k,\Lambda^k)-\mathcal{L}_\lambda(\mathbf{A}^{k+1},\mathbf{Z}^{k+1},\mathbf{q}^{k},\mathbf{B}^k,\Lambda^k) \\
= & \ell_{\mathcal{S}_1}(\mathbf{Z}^{k})+P^{k+1}(\mathbf{Z}^{k}) - \ell_{\mathcal{S}_1}(\mathbf{Z}^{k+1}) - P^{k+1}(\mathbf{Z}^{k+1}) \\
\geq &\frac{1}{2}(\frac{1}{\eta_z}-\hat{L}_{Z})\|\mathbf{Z}^{k+1}-\mathbf{Z}^{k}\|^2,
\end{align*}
which completes the proof.
\end{proof}




\end{appendices}


\bibliography{references}

\end{document}